\documentclass{article}

\usepackage{amsmath,amssymb,amsthm}
\usepackage{stmaryrd}
\usepackage{bussproofs}
\usepackage{mathpartir}
\usepackage{hyperref}
\usepackage{amsfonts}
\usepackage{breqn}

\usepackage{multicol}
\usepackage{multirow}
\usepackage{array}
\usepackage{booktabs}

\usepackage[dvipsnames,svgnames,table]{xcolor}

\usepackage{pdfpages}
\usepackage{graphicx}
\graphicspath{ {./figuras/}}
\usepackage{tikz}
\usetikzlibrary{fit}
\usepackage{wrapfig}

\usepackage{mathrsfs}

\usepackage{placeins} 
\usepackage{makecell} 
\usepackage[title,titletoc]{appendix} 
\usepackage{enumitem} 

\newtheorem{theorem}{Theorem}[section]
\newtheorem{definition}[theorem]{Definition}

\newtheorem{lemma}[theorem]{Lemma}
\newtheorem{corollary}[theorem]{Corollary}

\newtheorem{properties}[theorem]{Properties}
\newtheorem{remark}[theorem]{Remark}

\newcommand\ooplus{\mathbin{\ooalign{$\scriptstyle\bigcirc$\cr\hidewidth$\scriptstyle\oplus$\hidewidth}}}
\newcommand\ooplusg{\mathbin{\ooalign{$\bigcirc$\cr\hidewidth$\oplus$\hidewidth}}}
\newcommand\ootimes{\mathbin{\ooalign{$\scriptstyle\bigcirc$\cr\hidewidth$\scriptstyle\otimes$\hidewidth}}}

\newcommand{\expfi}{\Theta^\oplus_{\varphi}}
\newcommand{\exppsi}{\Theta^\oplus_{\psi}}

\newcommand{\excedfi}{\Theta \top \varphi}
\newcommand{\excedpsi}{\Theta \top \psi}
\newcommand{\excedlor}{\Theta \top \varphi \lor \psi}

\newcommand{\papfi}{(\Theta, \varphi)^{\scriptscriptstyle{\mathbb{APE}}}}

\newcommand{\revfi}{\Theta^\otimes_{\varphi}}

\newcommand{\til}{\mathord{\sim}}

\newcommand{\expfip}{\Theta^{\ooplus}_{\varphi}}
\newcommand{\exppsip}{\Theta^{\ooplus}_{\psi}}
\newcommand{\expfilorp}{\Theta^{\ooplus}_{\varphi \lor \psi}}
\newcommand{\postp}[1]{(\Theta^{\ooplus}#1)}

\newcommand{\papfip}{(\Theta, \varphi)^{\scriptscriptstyle{\mathbb{APE}}}_\circ}

\newcommand{\papfig}{(\Theta, \varphi)^{\scriptscriptstyle{\mathbb{APE}}}_p}
\newcommand{\pappsig}{(\Theta, \psi)^{\scriptscriptstyle{\mathbb{APE}}}_p}
\newcommand{\paplorg}{(\Theta, \varphi \lor \psi)^{\scriptscriptstyle{\mathbb{APE}}}_p}

\newcommand{\revfipi}{\Theta^{\ootimes_i}_{\varphi}}
\newcommand{\revfipe}{\Theta^{\ootimes_e}_{\varphi}}

\setlist[description]{noitemsep} 

\title{AGM-like Paraconsistent Partial Meet Abductive Expansion Operation}

\author{%
    Ulisses Franceschi Eliano%
    \thanks{%
        The author is a PhD student under the supervision of Marcelo Esteban Coniglio 
        and supported by FAPESP (Fundação de Amparo à Pesquisa do Estado de São Paulo) 
        -- process \mbox{2024/22555-9}.%
    } \\
    Institute of Philosophy and the Humanities (IFCH) and \\
    Centre for Logic, Epistemology and The History of Science (CLE), \\
    University of Campinas (UNICAMP) \\
    \texttt{u225017@dac.unicamp.br}
}

\date{\today}

\date{\today}

\begin{document}

\maketitle

\begin{abstract}
In his 1996 doctoral thesis \cite{Pagnucco}, Maurice Pagnucco created the first \textbf{AGM}-like abductive expansion operation. Taking his operation as a basis, as well as a taxonomy---inspired by Atocha Aliseda \cite{Aliseda}---responsible for highlighting and formalizing the main components of abductive reasoning, the main aim of this paper is to present a new \textit{paraconsistent} \textbf{AGM}-like abductive expansion operation---capable of assimilating contradictory explanatory hypotheses without trivialization and the consequent absurd epistemic state---, with its postulates and its transitively relational \textit{partial meet} construction. To a large extent, the formal development presented in this paper was only made possible by the recent creation of the paraconsistent logic \textbf{RCbr} \cite{TESTA_CON_MARTIN_CBR}, an \textbf{LFI} (\textit{Logics of Formal Inconsistencies}) that establishes properties especially relevant to belief revision contexts, in particular, the ability to be \textit{self-extensional}---i.e., to satisfy the \textit{replacement} property. This is the first of two papers: the paraconsistent abductive expansion operation announced here---which is part of a new system called \textbf{AGM}$p_{abd}$---despite bringing many interesting features, does not assign any relevant epistemic role to the paraconsistent operators of negation $\neg$ and consistency $\circ$. Only in a second paper will an analogous paraconsistent abductive expansion operation---which is part of another new system, \textbf{AGM}$\circ_{abd}$---be enhanced in this direction. Nevertheless, to the best of my knowledge, the operation developed in this paper is the first of its kind in the \textbf{AGM} literature.
\end{abstract}

\section{Introduction}
\label{introduction}

Charles Sanders Peirce is considered the founder of abduction as a third type of reasoning, independent of deduction and induction. Abduction is the reasoning responsible for formulating and selecting hypotheses capable of explaining, albeit rudimentarily and conjecturally, intriguing and surprising facts. Or, according to Fann, ``[...] Peirce's theory of abduction is concerned with the reasoning which starts from data and moves towards hypotheses''\cite[p.5]{Fann}. Peirce's theory of abduction is philosophically rich and is related to many other concepts in his philosophy. For the purposes of this article, however, I would like to highlight four important aspects of abduction: (i) the \textit{creative aspect}, (ii) the \textit{logical-inferential aspect}, (iii) its \textit{distinction from the other two types of reasoning}, and (iv) the \textit{selective aspect}. The creative aspect of abduction (i) manifests itself when an inquiry agent, faced with a surprising fact, originally conceives hypotheses to explain it, as an instinctive act\footnote{Peirce, in fact, understood abduction as a kind of \textit{guessing instinct}. In his words, abduction ``[...] tries what \textit{il lume naturale}, which lit the footsteps of Galileo, can do. It is really an appeal to instinct'' (CP 1.630) and abduction is ``[...] nothing but guessing'' (CP 7.219).}, and, thus, \textit{expands} its initial informational repertoire with new conceptions. In Peirce's words, ``Abduction is originary in respect to being the only kind of argument which starts a new idea'' (CP 2.96)\footnote{The acronym ``CP'' refers to the \textit{Collected Papers} of Charles Sanders Peirce, which can be found in the references as \cite{PeirceCP1} and \cite{PeirceCP2}.}. The creative originality of this type of reasoning places it at the center of the debate about the possibility of a logic of discovery, in Hans Reichenbach's well-known terms\footnote{The distinction between \textit{context of discovery} and \textit{context of justification} is well known in the literature. The author considered, for example, Einstein's new theory of gravitation a great discovery, not because the physicist was lucky or had a hunch, but because ``[...] the known facts indicate such a theory; i.e., that an inductive expansion of the known facts leads to the new theory. This is precisely what distinguishes the great scientific discoverer from a clairvoyant''\cite[p.382]{Reich}. Reichenbach, however, identified this type of reasoning involved in scientific discovery, like many authors of his time (and today), with induction.}, something that goes beyond the scope of this article. Despite being creative, abduction has a logical-inferential aspect (ii), which seems paradoxical\footnote{This is a well-known paradox in the specialized literature and is called the \textit{abduction paradox}. Unfortunately, this interesting topic will not be developed in this document. To read more about this, including an excellent attempt to resolve it, see \cite{Santaella}.}. The Peircean conception of inference\footnote{In Peirce's words, ``It must be remembered that abduction, although it is very little hampered by logical rules, nevertheless is logical inference, asserting its conclusion only problematically or conjecturally, it is true, but nevertheless having a perfectly definite logical form'' (CP 5.188).}, however, bears little or no resemblance to the notion of logical consequence, syntactical, or semantic, which has become the canonical standard in the orthodoxy of contemporary logic. It is not possible to address Peirce's inferentialism here, however, it is my understanding that an inference, in the Peircean general sense, but in particular the abductive one, is best represented by an \textbf{AGM}-style belief change operation embedded in the \textit{abductive process} depicted in the taxonomy in Section \ref{taxogeral}

Furthermore, it is important to keep in mind that Peircean abduction should not be confused with the other two types of reasoning (iii). The distinction between abduction and deduction is clearer and does not usually generate controversy, but abduction is often confused with induction. It is not infrequently considered a type of induction\footnote{This confusion seems to have its origins in Gilbert Harman's influential article, \textit{The Inference to the Best Explanation}. I think that Harman's article, in light of Peirce's original abduction theory, deserves much criticism (and it is not uncommon in the literature), but I will not go into it in this article.}. In Peirce's words, however, ``I don't think the adoption of a hypothesis on probation can properly be called induction; and yet it is reasoning, and though its security is low, its uberty is high'' (CP 8.388). Abduction, in turn, is the only truly ampliative inference, in Peirce's most mature conception, and what differentiates it from induction is that the latter ``[...] infers the existence of phenomena such as we have observed in cases which are \textit{similar}, while hypothesis supposes something of a \textit{different kind} from what we have directly observed, and frequently something which it would be impossible for us to observe directly'' (CP 2.640). While abduction is the ``[...] first step of scientific reasoning'' (CP 7.218), and therefore seeks to formulate a new theory, induction, on the other hand, is only responsible for evaluating the experimental results as the final step of inquiry, without actually adding anything new. Abduction also selects (iv) the ``best'' hypotheses from among the countless\footnote{The incountability of possible explanatory hypotheses, that is, ``[...] the idea that data underdetermine theory - that an infinite variety of alternative hypotheses can conform equally well to any finite body of empirical data'' \cite[p.72]{Rescher2}, seems to be a commonplace in the philosophy of science and Peirce seemed, in his time, to endorse such a view. In his words, ``[...] Think of what trillions of trillions of hypotheses might be made of which one only is true; and yet after two or three or at the very most a dozen guesses, the physicist hits pretty nearly on the correct hypothesis'' (CP 5.172).} formulated ones. The philosophical and scientific literature concerning the criteria that adequately apprehend the notion of ``best'' or more ``preferable'' in this context is vast. The passage in CP 7.220, in which Peirce cites the three criteria for selecting hypotheses, is well known. They are: the hypotheses (a) must, in fact, \textit{explain} the surprising facts, (b) must be capable of being subjected to \textit{experimental verification}, and (c) must be conditioned by the \textit{principle of economy of research}. Each of these criteria requires considerable philosophical exposition, including the doctrine of Pragmatism, which is closely linked to abduction, a topic that, unfortunately, cannot be adequately addressed in this article. In terms of contemporary logic (but not Peircean conception of logic), these criteria can be considered ``extra-logical'', and so I will refer to them in this paper, more generally, as the \textit{pragmatic and explanatory principles} of hypothesis selection.

To the best of my knowledge, Peirce did not address the question of the possibility of the coexistence of contradictory hypotheses. It seems to me, however, that they can indeed coexist in a wide variety of contexts, as argued by the authors, for example, in \cite{Bueno_Car_Con_Ab_2017} and \cite{Bueno_Car_Con_Ant_LET_2022}. A paraconsistent approach is therefore necessary to formally represent such situations. The paraconsistent abductive expansion operation presented in this paper will be referred to as the \textbf{AGM}$p_{abd}$ abductive expansion operation, the system of which it is a component\footnote{\label{rodape_agmp_agmo}In Testa's work \cite{Testa}, later improved in \cite{Testetal}, two systems were created: \textbf{AGM}$p$ and \textbf{AGM}$\circ$. The former establishes paraconsistent contraction and revision operations in which no epistemic role is, in fact, assigned to the paraconsistent operators $\neg$ and $\circ$. In the latter, the operators gain more epistemically relevant roles, especially the consistency operator $\circ$, which confers the \textit{irrefutability} of a given belief $\varphi$, when $\varphi \in \Theta$ and $\circ \varphi \in \Theta$, for $\Theta = Cn(\Theta)$---in this case, $\varphi$ is said to be \textit{boldly accepted} in $\Theta$. In turn, two \textbf{AGM}-like abductive systems, with their respective operations, postulates and constructions, are being developed in my ongoing PhD research, and the names \textbf{AGM}$p_{abd}$ and \textbf{AGM}$\circ_{abd}$ were analogously adopted. This paper concerns only the abductive expansion operation that is part of the \textbf{AGM}$p_{abd}$ system.}. As is conventional in the \textbf{AGM} literature, especially following the publication of Gärdenfors and Rott's article \cite{Gard_Rott}\footnote{As is well known in the literature, the criteria established by the authors are: \textit{consistency}---which, in the classical case, is the same as non-contradiction---, \textit{deductive closure}, \textit{minimality} and \textit{epistemic entrenchment}.}, \textbf{AGM} operations should be guided by \textit{rationality criteria}. For the \textbf{AGM}$p_{abd}$ system the rationality criteria are the following:

\begin{description}

    \item (i) \textit{Non-triviality}: an epistemic state must be \textit{always} non-trivial;
    \item (ii) \textit{Deductive closure}: any belief deductively implied by other beliefs of an epistemic state must be included in the epistemic state;
    \item (iii) \textit{Minimality}: when discarding information, the maximum amount of information from the original state should be preserved;
    \item (iv) \textit{Epistemic entrenchment}: epistemically more valuable beliefs (or ``more entrenched'') should be more resistant to change;
    \item (v) \textit{Information acquisition}: the agent is interested in acquiring new and valuable error-free information;
     
    \leftskip 20pt
          \item (v.a) \textit{Surprising fact}: the agent, when faced with a surprising fact, \textit{always} initiates an \textit{abductive process} to elaborate and incorporate explanatory hypotheses;
        \item (v.b) \textit{Explanatory adequacy}: the new hypotheses must be \textit{properly explanatory} for the surprising fact;
          \item (v.c) \textit{Informational balance}: the agent must achieve a balance between the amount of new hypotheses to be incorporated and the risk of error, depending on their degree of caution or bouldness;       
        \item (v.d) \textit{Abductive entrenchment}: some hypotheses are more significant and relevant than others, and the agent selects them according to \textit{explanatory and pragmatic principles};
   
\end{description}

Criterion (i) of non-triviality differs from the classical \textbf{AGM} system in at least two respects: first, the specific requirement of \textit{non-triviality}, rather than \textit{consistency}. This is a fundamental distinction in the paraconsistency literature, since, unlike \textbf{CPL} (\textit{Classical Propositional Logic}), in which the terms ``inconsistency'', ``trivialization'' and ``contradiction'' are synonymous (and formally equivalent), these notions do not coincide in an \textbf{LFI}, as we shall see in Section \ref{rcbr_logic}. The second aspect concerns the term ``always''. Unlike the classical \textbf{AGM} system, in which the expansion operation does not formally block initial trivial epistemic states, nor does it prohibit expansion towards triviality, in the \textbf{AGM}$p_{abd}$ abductive expansion there is no room for triviality. Criterion (ii) of deductive closure remains the same as in classical \textbf{AGM}, i.e., epistemic states will be modelled as deductively closed belief sets. Criteria (iii) and (iv) are old acquaintances of the \textbf{AGM} literature and are only present here to guide other operations of the \textbf{AGM}$p_{abd}$ system that do not strictly concern the paraconsistent abductive expansion operation, but which are worth keeping as general guiding criteria for the system\footnote{\label{nota_rodape_minimalidade_revisao}Regarding criterion (iii) of minimality, three things should be noted. 1) In his inaugural thesis \cite{Pagnucco}, Pagnucco also develops the abductive revision operation, based on the Levi Identity, but with the Levi saturated contraction, followed by his abductive expansion operation, i.e., $\revfi = (\Theta^-_{\til \varphi})^\oplus_\varphi$. This is a simpler operation---quite limited due to the non-monotonic nature of the abductive expansion operation and the impossibility of establishing an adequate inclusion relation between abductive expansion and revision (see \cite[p.165-166]{Pagnucco})---but which, because of the prior contraction operation, needs to be guided by the minimality criterion. The same occurs in the \textbf{AGM}$p_{abd}$ system, with the difference that, since it is a paraconsistent system, an \textit{external revision} operation---i.e., obtained by means of the Reverse Levi Identity, in which expansion occurs \textit{before} contraction, i.e., $\revfipe = (\expfip)^-_{\til \varphi}$---can also be obtained, although suffering from the same (and others) limitations. Unfortunately, revision operations are not the subject of this paper. 2) Moreover, Pagnucco also provides a semantic interpretation of minimality via possible worlds and Grove spheres (see \cite[p.114-121]{Pagnucco}). I will not delve into this topic in this paper. 3) Furthermore, notoriously, the minimality criterion contrasts with criterion (v), which seems more appropriate for abductive expansion, as we shall see. I also think that, specifically for abductive expansion, the minimality criterion should be subordinated to other foundations, of an ``extra-logical'' nature, according to \textit{pragmatic principles} of hypothesis selection and preference---especially those found in C.S. Peirce's own work, such as \textit{simplicity, explanatory power, economy, experimental verification, etc.}---which are not within the scope of this paper. This issue, however, will be taken up again in the second paper, concerning the \textbf{AGM}$\circ_{abd}$ system.}\footnote{In the case of criterion (iv), as in the previous case, it is prudent to keep it, since it concerns the abductive revision operations---in fact, the contraction operation underlying them. Epistemic entrenchment, however, will have no role in the abductive expansion operation proper.}.

Criterion (v) is an inheritance from Pagnucco's own system, insofar as the author reinterprets it from Isaac Levi's work \cite[p.72]{Levi}. This criterion indicates the balance, desired by an ideal rational agent, between the interest in acquiring new and valuable information, in this case through abduction, and the risk of falling into error\footnote{It is important to note, however, that depending on the logic underlying the system, whether classical or paraconsistent, this criterion may mean different things. On the one hand, as we shall see, paraconsistency allows the acquisition of new information in the abductive process in cases where the classical one does not. On the other hand, the term ``error'', in the classical case, means inconsistency, in the sense of contradiction. In the paraconsistent case, the term ``error'' comes to be interpreted exclusively as triviality. Contradiction, therefore, no longer enters the agent's rational balance as an error and comes to hold an informational value that deserves due attention.}. I consider this criterion, however, sufficiently general to encompass new criteria, which determine \textit{how} the agent will carry out such informational acquisition. Thus, criteria (v.a) and (v.b) capture important aspects of abductive reasoning proper that are very prominent in C.S. Peirce's work, as we just discussed. First, criterion (v.a) tells us that abductive reasoning is a process that arises from a surprising fact, which impresses upon us a troubling doubt, and which needs to be explained. As such, it is ampliative par excellence---in fact, the \textit{only} truly ampliative reasoning. The emphasis on the term ``always'' is therefore important, since there is no abductive reasoning that leaves the initial epistemic state unchanged. Criterion (v.b) minimally qualifies the hypotheses, accordingly to Peircean thought, in the sense of requiring them to be properly explanatory. These two criteria justify the creation of a specific taxonomy for abduction, which will be addressed in Section \ref{taxogeral}, and, to a large extent, in addition to paraconsistency itself, differentiate the abductive expansion operation developed here from Pagnucco's operation.

Finally, criteria (v.c) and (v.d) can also be found, in a different exposition, in Pagnucco's thesis. Criterion (v.c) tells us that a balance must be achieved between an extremely self-confident agent, i.e., one who elaborates too many hypotheses and thus may more easily fall into error, and, on the other hand, a very sceptical agent, who does not run many risks of incorporating errors, but, to the same extent, does not elaborate enough hypotheses and therefore fails to incorporate relevant information. The operation that performs this balance is precisely the \textit{partial meet} abductive expansion that will be developed in Section \ref{agmpabd_expansion_construction}\footnote{In his thesis \cite{Pagnucco}, Pagnucco also presents \textit{maxichoice} and \textit{full meet} abductive expansion operations, in line with the classical \textbf{AGM} literature. The same operations can also be established in the \textbf{AGM}$p_{abd}$ system, with many of the same characteristics already found by Pagnucco, among others. Unfortunately, it will not be possible to develop them in this paper.}. Criterion (v.d) attempts to capture the notions underlying \textit{betterness} relations between hypotheses, which are generally those used by the agent to select them. As we just discussed, Peirce himself presented many ``extra-logical'' criteria in this sense, the \textit{pragmatic and explanatory principles} of hypothesis selection. Formally, however, there are several possible ways of representing such notions\footnote{Pagnucco \cite{Pagnucco} also presented a construction called \textit{abductive entrenchment}, similar to the epistemic entrenchment of classical \textbf{AGM}, but interested, in the opposite way, in sentences that lie outside the belief set and need to be incorporated. The same construction can also be obtained in the \textbf{AGM}$p_{abd}$ system, with the appropriate formal changes. Once again, this is a construction that was only made possible thanks to the properties of the \textbf{RCbr} logic, which will be introduced in Section \ref{rcbr_logic}. The paraconsistent abductive expansion operation based on the abductive entrenchment order, however, will be the subject of another future paper.}. In this paper, the relationality and transitivity imposed on the \textit{partial meet} selection function will fulfil this role.

\section{A Taxonomy for Abducion}
\label{taxogeral}

Inspired by the taxonomy for abduction proposed by Atocha Aliseda\footnote{See \cite[46-48]{Aliseda}. It is important to note that Aliseda implements her abductive process specifically through abductive tableaux (despite the author establishing some interesting connections between her tableaux and the \textbf{AGM} system (see \cite[Chapter 8]{Aliseda}). The objective of this article, however, like Pagnucco's, is to present an AGM-style operation with its respective \textit{construction} for epistemic states expanded by explanatory hypotheses, not to \textit{construct} the hypotheses themselves via tableaux, as proposed by Aliseda. Furthermore, Aliseda proposes a ternary format for her inferential parameter $\Theta ~| ~\varphi \Rightarrow \alpha$ (see \cite[p.68]{Aliseda}) and various abductive styles (see \cite[74]{Aliseda}). In the first case, Aliseda's objective is to study the behavior of the well-known structural rules of Gentzen-style sequent calculus when applied to different kinds of premises (as required by the Hempel-Oppenheim DN model), which is not my objective in this article. \pagebreak In the second case, the adequate explanatory inference proposed here contains exactly the same restrictions, jointly, as the consistent and explanatory styles proposed by the author. My inspiration in Aliseda's work, however, is justified by recognizing, in her taxonomy, a very interesting theoretical conception general enough to address various implementations of abductive reasoning, provided it is thought of in terms of process, not merely as a simple logical form.}, a taxonomy for abduction is a general systematization of three components of an \textit{abductive process}: (i) the \textit{triggers}, (ii) an \textit{adequate explanatory inference}, and (iii) the \textit{outputs}. Let us therefore look at each of the components.

\

\textbf{\textit{(i) Triggers}}

\

The triggers represent the beginning of the abductive process, that is, they characterize the \textit{initial surprise} in the face of an unusual and surprising phenomenon or fact $\varphi$. From a formal point of view and considering initially classical propositional logic $\mathbb{L} = \langle \mathcal{L}_\Sigma, \vdash \rangle$, in witch formulas in $\mathcal{L}$ are generated by classical propositional signature $\Sigma = \{\land, \lor, \to, \til, \bot, \top \}$. In this case, the simpler notation $\mathbb{L} = \{\textbf{CPL}\}$ will be used. Let $\Theta$ be a set of formulas representing agent's previous theories or beliefs and $\Theta \cup \{\varphi\} \subseteq \mathcal{L}_\Sigma$. The (classical) triggers for an abdutctive process are:

\begin{description} \label{gatilhos_abducao}

\item (i) \textit{Abductive novelty}: $\varphi$ is a novelty if $\Theta \nvdash_{\mathbb{L}} \varphi$ and $\Theta \nvdash_{\mathbb{L}} \til \varphi$
\item (ii) \textit{Abductive anomaly}: $\varphi$ is an anomaly if $\Theta \nvdash_{\mathbb{L}} \varphi$ and $\Theta \vdash_{\mathbb{L}} \til \varphi$
\end{description}

The requirement $\Theta \nvdash_{\mathbb{L}} \varphi$ is the same in both triggers. It is justified, in the abduction context, because, if $\Theta \vdash_{\mathbb{L}} \varphi$, there is no surprise and, thus, no abducive reasoning. The difference between both triggers, however, is in the fact that, in an abductive novelty, the agent don't have any initial clue about the observed fact, and in the abductive anomaly the fact is inconsistent---in a classical sense, contradictory---with his previous beliefs.

\

\textbf{\textit{(ii) Adequate explanatory inference}}

\

Let us consider the following definition.

\begin{definition}[Abductive explanation]
\label{explicacao_abdutiva}

Let $\mathbb{L} = \{\textbf{CPL}\}$ and $\Theta \cup \{\varphi\} \cup \{\alpha\} \subseteq \mathcal{L}_\Sigma$. The sentence $\alpha$ is an abductive explanation of $\varphi$ with respect to $\Theta$ if and only if:
\end{definition}

\begin{description}

\item (i) $\Theta \cup \{\alpha\} \vdash_{\mathbb{L}} \varphi$;
\item (ii) $\Theta \cup \{\alpha\} \nvdash_{\mathbb{L}} \bot$ ;
\item (iii) $\Theta \nvdash_{\mathbb{L}} \varphi$;
\item (iv) $\{\alpha\} \nvdash_{\mathbb{L}} \varphi$.
\end{description}

Conditions (i) and (ii) are quite common in the literature. However, for the explanatory aspect of a hypothesis $\alpha$, if $\Theta \cup \{\alpha\} \vdash_{\mathbb{L}} \varphi$ and $\Theta \cup \{\alpha\} \vdash_{\mathbb{L}} \bot$ are necessary\footnote{It is well known in the literature, since Aristotle's \textit{Posterior Analytics}, that deduction is not sufficient for explanation, but it is necessary (see \cite{Angioni} e \cite{Angioni2}). There are many other intensional aspects to consider when \textit{explanatory relevance} and \textit{explanatory/causal direction} are in question. They simply can not be captured by deduction alone. The problem of scientific explanation has survived through time until it was revived in the 20th century by Hampel and Oppenheim and their well-known DN-model \cite{Hempel-Oppenheim}. If not the most successful contemporary approach to scientific explanation, it has certainly become paradigmatic, especially in the formal logic field. Of course, there are many others contemporary approaches in philosophy that even disconsider deduction as a necessary condition for explanation, for instance, \cite{Salmon} and \cite{VanFraassen}.}condition for explanation, in Hempelian terms at least, they are not sufficient\footnote{Of course, philosophically, Hempel and Oppenheim \cite{Hempel-Oppenheim} are very distant from Peircean philosophy and his Pragmatism. The philosophical notions of cause, laws of nature and theories (and, of course, explanation) are very different between these two schools of thought. Several 20th-century logical positivists and philosophers of science, including Hempel and Oppenheim themselves, were influenced by previous interpretations, dating back to the 19th century, of the Humean philosphy that recognized him as a causal regularist. I cannot go further on this topic, but it is important to say that, when the main motivation is the logical formalization of the explanation (thus disregarding strictly statistical and computational approaches), the work of Hempel and Oppenheim seems to be, to a large extent, inescapable.}. I claim that the restrictions $\Theta \vdash_{\mathbb{L}} \varphi$ and $\{\alpha\} \vdash_{\mathbb{L}} \varphi$ must also be considered as \textit{minimum} conditions to achieve \textit{some}\footnote{Note that Hempel and Oppenheim give us many logical and empirical conditions to guarantee explanatory relevance  (see \cite[247-249]{Hempel-Oppenheim}). I think that their definition of \textit{potential explanans} (\cite[277-278]{Hempel-Oppenheim} can be considered a desirable general goal for all those who desire a good logical-deductive representation of explanation, at least when abduction and its potential hypotheses are in question. Unfortunately, Classical Propositional Logic, used predominantly in AGM-like implementations, makes this complete adequacy to their definition unfeasible, as the authors require quantified first-order language to represent laws or background theories. For the purposes of this paper, however, I think these proposed restrictions are good enough.} explanatory relevance over the fact $\varphi$: the appropriate explanation for $\varphi$ must be an interplay of the background theories $\Theta$ together with the hypothesis $\alpha$, and $\alpha$ can not manifests itself as a total self-explanation\footnote{A self-explanation is an empirical and logical problem that emerges when deduction is underneath. Let $\Theta$ be background theories and $\varphi$ an \textit{explanandum}. So, due to the deduction logical properties, it is always possible to create explanations of the kind $\varphi \vdash \varphi$. If reflexivity is a desirable property for deduction, certainly it is not for explanation.}, witch is, in this context, represents the blocking of the reflexivity property, present in all Tarskian deductive logic.

\

\textbf{\textit{(iii) Outputs}}

\

If the triggers are the beginning of the process, the outputs are the end, the final result of the process. In the case of the formal system presented by Aliseda \cite{Aliseda}, based on abductive tableaux responsible for ``generating'' or ``constructing'' the abductive hypotheses, the outputs of the process are the multiple abductive hypotheses $\alpha$ themselves. Differently, as in Pagnucco's case \cite{Pagnucco}, my proposal is to represent abductive reasoning by means of \textbf{AGM}-like operation. Thus, a belief change operation is a function that takes an initial epistemic state, together with a sentence of the language---in this case, a sentence representing the phenomenon to be explained---and returns multiple possible expanded epistemic states, proportionally to the number of available explanatory hypotheses. In other words, the outputs of the abductive process, instead of the abductive hypotheses, due to the nature of \textbf{AGM}-style operations, are the expanded epistemic states themselves relative to them\footnote{It is important to emphasize, therefore, that unlike the abductive tableau obtained by Aliseda (and also the paraconsistent abductive tableaux found in \cite{Bueno_Car_Con_Ab_2017} and \cite{Bueno_Car_Con_Ant_LET_2022}), which is responsible for ``constructing'' the hypotheses---which can be interpreted as a formal attempt to represent the ``creative'' aspect of abduction---the paraconsistent abductive expansion operation presented in this paper already starts from constructed hypotheses, as can be noted in the very definition \ref{conj_hip_abdutivas_agmpabd} in Section \ref{agmpabd_expansion_postulates}. In my view, therefore, these are complementary formal implementations, not competing ones.}.
\section{The RCbr Paraconsistent Logic}
\label{rcbr_logic}

\textbf{RCbr} is a recently developed \textbf{LFI} \cite{TESTA_CON_MARTIN_CBR} specifically designed to serve as the underlying logic for the paraconsistent belief revision \textbf{AGM}$\circ$ system\footnote{As already mentioned in footnote \ref{rodape_agmp_agmo}, the \textbf{AGM}$\circ$ is a paraconsistent belief revision system first developed in \cite{Testa} and extensively improved in \cite{Testetal} and \cite{TESTA_CON_MARTIN_CBR}. This system can establish four new AGM-style epistemic attitudes---besides the three well-known from standard \textbf{AGM}---relative to the agent's belief \textit{strength} and contradictoriness, a new AGM-like paraconsistent contraction operation, and other paraconsistent internal and (now possible) external revision operations. In \cite{Testa}, due to general \textbf{LFI} properties, only paraconsistent non-extensional \textit{partial meet} contraction and revision operations were developed, and in \cite{Testetal}, the authors obtained an \textit{extensional} version of the same operations. Only with \textbf{RCbr} could an epistemic entrenchment order be established and proven to satisfy the paraconsistent contraction postulates.}. To describe \textbf{RCbr}'s syntax, semantics, and general properties, some brief background is necessary. The main idea underlying \textbf{LFI}s is the distinction between trivialization and contradiction, where explosion---due to the classical \textit{ex falso} principle---occurs only in a \textit{controlled} manner when the contradictory proposition in question is marked as consistent by means of the consistency operator $\circ$. Thus, in \textbf{LFI}s:

\begin{center}
$\varphi, \neg \varphi \nvdash_{LFI} \psi$, but $\circ \varphi, \varphi, \neg \varphi \vdash_{LFI} \psi$
\end{center}

It is well known in the paraconsistency literature\footnote{The most comprehensive guide to \textbf{LFI}s---particularly to \textbf{mbC} and many of its extensions---can be found in \cite{CONCAR_2016}.} that \textbf{mbC} is the weakest \textbf{LFI}. Therefore, let us consider a standard and supraclassical Tarskian logic\footnote{Definitions \ref{definicao_logica_tarskiana_standard} and \ref{definicao_logica_tarskiana_supraclassuca} found in Appendix \ref{apendiceA}. As so, the \textbf{mbC} logic also satisfies fundamental rules and properties of \textbf{CPL}, such as the \textit{deduction theorem}, \textit{disjunction of premises}, and \textit{proof by cases}, as well as $delta$-saturated properties and Lindenbaum-Łoś Theorem. See properties \ref{propriedades_teorema_da_deducao_disjuncao_das_premissas} and \ref{proposicao_propriedade_delta_saturado} and theorem \ref{teorema_lindenbaum_los} in Appendix \ref{apendiceA}.} \textbf{mbC} $= \langle \mathcal{L}_{\Sigma_\circ}, \vdash \rangle$ over the language $\mathcal{L}_{\Sigma_\circ}$ generated by the signature $\Sigma_\circ = \{\land, \lor, \to, \neg, \circ \}$, where $\neg$ denotes the paraconsistent negation and $\circ$ the primitive consistency operator. Let us also define \mbox{$\bot_\psi =_{\text{def}} \psi \land \neg \psi \land \circ \psi$} as the \textit{falsum} particle and $\sim_\psi \! \varphi =_{\text{def}} (\varphi \to \bot_\psi)$ as the defined classical (strong) negation. Consider the following Hilbert-style axiomatic scheme:

\begin{description}
    \item[(1)] All \textbf{CPL}$^+$ basic axioms for $\land, \lor, \to$ plus \textit{Modus Ponens}
    \item[(2)] $\varphi \lor \neg \varphi$
    \item[(3)] ($bc1$) $\circ \varphi \to (\varphi \to (\neg \varphi \to \psi))$
\end{description}

The derivation in \textbf{mbC} is defined as usual\footnote{See \cite[p.34]{CONCAR_2016}.}. Three main characteristics deserve attention: (i) the absence of the principle of non-contradiction, (ii) the excluded middle axiom 2, which is formulated using the paraconsistent negation $\neg$, and (iii) the axiom 3 ($bc1$), which governs the behavior of the consistency operator $\circ$. Semantic valuations for the classical connectives operate as usual, while specific non-deterministic valuations for $\neg$ and $\circ$ are provided below:

\begin{description}
    \item[($vNeg$)] $v (\neg \varphi) = 0  \implies v(\varphi) = 1$
    \item[($vCon$)] $v (\circ \varphi) = 1 \implies v(\varphi) = 0\;$ or $\;v(\neg \varphi) = 0$
\end{description}
\label{definicao_valoracoes_mbC}

The \textbf{mbC} fundamental properties\footnote{The most notable properties of \textbf{mbC} are: (i) a contradiction implies inconsistency, but not vice versa, that is, $\varphi \land \neg \varphi \vdash_{\textbf{mbC}} \neg \circ \varphi$, but $\neg \circ \varphi \nvdash_{\textbf{mbC}} \varphi \land \neg \varphi$; (ii) consistency implies non-contradiction, but not vice versa, that is, $\circ \varphi \vdash_{\textbf{mbC}} \neg (\varphi \land \neg \varphi)$, but $\neg (\varphi \land \neg \varphi) \nvdash_{\textbf{mbC}} \circ \varphi$; finally, since most classical equivalences involving negation do not hold in general, \textbf{mbC}---like most \textbf{LFI}s---is not \textit{self-extensional}; that is, it does not satisfy the \textit{replacement} property, which we will examine shortly. } and its soundness and completeness theorems are well-established\footnote{See \cite[p.36-38]{CONCAR_2016}.}. There are many \textbf{mbC} extensions in the literature. The following extensions and axioms---where those on the right incorporate the adjacent ones on the left---are relevant for our purposes:

\begin{table}[h!]
    \centering
    \resizebox{\textwidth}{!}{
    \begin{tabular}{|c|c|c|c|c|c|c|}
         \hline
         \cellcolor{gray!5}\textbf{mbC} & \cellcolor{gray!10}$\subset$ & \cellcolor{gray!15}\textbf{mbCciw} & \cellcolor{gray!20}$\subset$ & \cellcolor{gray!25}\textbf{Cbr} & \cellcolor{gray!30}$\subset$ & \cellcolor{gray!35}\textbf{RCbr}\\
         \hline
         \cellcolor{gray!5}\textbf{mbC axioms} & \cellcolor{gray!10}$+$ & \cellcolor{gray!15}(ciw) $\circ \varphi \lor (\varphi \land \neg \varphi)$ & \cellcolor{gray!20}$+$ & \cellcolor{gray!25}\makecell{(ce) $\varphi \to \neg \neg \varphi$\\ (cf) $\neg \neg \varphi \to \varphi$} & \cellcolor{gray!30}$+$ & \cellcolor{gray!35}\textit{replacement}\\
         \hline
    \end{tabular}}
    \caption{\textbf{RCbr} axiomatic hierarchy.}
    
\end{table}

The \textbf{mbCciw} logic is the minimal extension of \textbf{mbC} that guarantees that the truth values of $\varphi$ and $\neg \varphi$ completely determine the truth value of $\circ \varphi$\footnote{\label{rodape_mbcciw_properties}Other important \textbf{mbCciw} properties are: (i) strongly denying a formula is equivalent to weakly denying it when the formula is considered consistent, that is, $\;\sim\!\varphi \equiv_{\textbf{mbCciw}}$ $\circ \varphi \land \neg \varphi$; and (ii) consistency operator can be defined in terms of other connectives, that is, $\circ \varphi \equiv_{\textbf{mbCciw}} \; \sim \!(\varphi \land \neg \varphi)$.}. The extension \textbf{Cbr} incorporates both double negation axioms and thus satisfies $\varphi \equiv_{\textbf{Cbr}} \neg \neg \varphi$. The semantics for \textbf{Cbr} is obtained through an algebraic swap structure-based three-valued non-deterministic matrix (\textit{Nmatrix}) as follows\footnote{The algebraization and the single finite matrix characterization of logical systems is often desirable. However, the problem of algebraizing \textbf{LFI}s, in the sense of Blok and Pigozzi \cite{BLOK_PIGOZZI_1989}, and obtain its respective single finite matrix characterization is well-known. Although algebraization in this sense can be obtained for a whole class of \textbf{LFI}s based on a three-valued finite matrix (see \cite[Section 4.4]{CONCAR_2016}), in general, many \textbf{LFI}s cannot be algebraitized in this sense or characterized by a single finite deterministic matrix (see \cite[Sec. 4.2]{CONCAR_2016})---including \textbf{mbC} and all its extensions discussed in this document. Nevertheless, they can be characterized by a single finite \textit{non-deterministic} matrix, as demonstrated through a class of multialgebras called \textit{swap structures} (see \cite[Ch. 6]{CONCAR_2016}, \cite{CON_FIB_SWAP}, and \cite{CON_ORELLANO_GOLZIO_SWAP}). In the article in question \cite{TESTA_CON_MARTIN_CBR}, therefore, the authors use these previously obtained results to derive the following Nmatrix. I will not detail such methods here, as I believe assuming the obtained results as valid is sufficient for this exposition.}:
 
\begin{definition}[Nmatrix for \textbf{Cbr}]
    Let $\mathcal{M}_{Cbr}$ be a three-valued non-deterministic matrix \\ $\langle \mathcal{T}, \mathcal{D}, \{\hat{\land}, \hat{\lor}, \hat{\to}, \hat{\neg}, \hat{\circ}\} \rangle$ over the signature $\Sigma_\circ = \{\land, \lor, \to, \neg, \circ \}$ with domain $\mathcal{T} = \{1, \frac{1}{2}, 0\}$ and set of designated values $\mathcal{D} = \{1, \frac{1}{2}\}$, such that the truth tables associated with each connective are as follows:
    \label{definicao_nmatriz_cbr_cie}
\end{definition}

   \begin{table}[h!]
   \centering
        \resizebox{\textwidth}{!}{
        \begin{tabular}{| c | c | c | c | c | c | c | c | c | c | c | c | c | c | c | c | c |  c | c | c |}

        \hline
        
        \multicolumn{20}{|c|}{$\mathcal{M}_{Cbr}$} \\

        \hline
     $\hat{\land}$ & 1 & $\frac{1}{2}$ & 0 & & $\hat{\lor}$ & 1 & $\frac{1}{2}$ & 0 & & $\hat{\to}$ & 1 & $\frac{1}{2}$ & 0 & & &$\hat{\neg}$ & & & $\hat{\circ}$\\

      \cline{1-4}  \cline{6-9}  \cline{11-14}  \cline{16-17} \cline{19-20} 
      \cline{1-4}  \cline{6-9}  \cline{11-14}  \cline{16-17} \cline{19-20} 

         1 & $\mathcal{D}$ & $\mathcal{D}$ & \{0\} & & 1 & $\mathcal{D}$ & $\mathcal{D}$ & $\mathcal{D}$ & & 1 & $\mathcal{D}$ & $\mathcal{D}$ & \{0\} & & 1 & \{0\} & & 1 & $\mathcal{D}$ \\
        
      \cline{1-4}  \cline{6-9}  \cline{11-14}  \cline{16-17} \cline{19-20} 

         $\frac{1}{2}$ & $\mathcal{D}$ & $\mathcal{D}$ & \{0\} & & $\frac{1}{2}$ & $\mathcal{D}$ & $\mathcal{D}$ & $\mathcal{D}$ & &  $\frac{1}{2}$ & $\mathcal{D}$ & $\mathcal{D}$ & \{0\} & & $\frac{1}{2}$ & $\{\frac{1}{2}\}$ & & $\frac{1}{2}$ & \{0\} \\

      \cline{1-4}  \cline{6-9}  \cline{11-14}  \cline{16-17} \cline{19-20} 

         0 & \{0\} & \{0\} & \{0\} & & 0 & $\mathcal{D}$ & $\mathcal{D}$ & \{0\} & & 0 & $\mathcal{D}$ & $\mathcal{D}$ & $\mathcal{D}$ & & 0 & \{1\} &  & 0 & $\mathcal{D}$ \\

      \cline{1-4}  \cline{6-9}  \cline{11-14}  \cline{16-17} \cline{19-20} 
    
    \end{tabular}}
   \caption{Nmatrix $\mathcal{M}_{Cbr}$ for \textbf{Cbr} connectives.}
    \label{tabela_valores_Nmatriz_Cbr_land_lor_to_neg}
 \end{table}

\FloatBarrier

Given that $\Gamma \models_{\mathcal{M}_{Cbr}} \varphi$ if for every valuation $v$ in the Nmatrix $\mathcal{M}_{Cbr}$\footnote{Where $v: \mathfrak{Fm} \to \mathcal{T}$ is considered a valuation in $\mathcal{M}_{Cbr}$ if $v(\varphi\#\psi) \in v(\varphi) \; \hat{\#} \; v(\psi)$ for $\# \in \{\lor, \land, \to\}$ and $v(\#\varphi) \in \hat{\#}v(\varphi)$ for $\# \in \{\neg, \circ\}$.}, it holds that for every $\psi \in \Gamma$, if $v(\psi) \in \mathcal{D}$, then $v(\varphi) \in \mathcal{D}$, the soundness and completeness theorems for \textbf{Cbr} can be established\footnote{Follows from \cite[Sec. 6.4]{CON_FIB_SWAP}.}. All these conditions are necessary---yet still not sufficient---for \textbf{RCbr} to satisfy three crucial properties for epistemic dynamics\footnote{Properties (i) and (ii) are desirable insofar as it is reasonable that an epistemic agent who considers a belief consistent should also consider its negation consistent; moreover, if the agent considers two beliefs logically equivalent, their consistency status should also be equivalent. The \textit{replacement} property, in turn, is a final necessary technical requirement to make \textbf{RCbr} suitable for an extensional \textbf{AGM} operation, as well as uniform substitutions.}:

\begin{properties} \mbox{}
\begin{description}
    \item[(i)]$\circ \varphi \equiv \circ \neg \varphi$;
    \item[(ii)] If $\varphi \equiv \psi$ and $\neg \varphi \equiv \neg \psi$, then $\circ \varphi \equiv \circ \psi$.
    \item[(iii)] (\textit{Replacement}) Given formulas $\varphi_i$ and $\psi_i$, for $1 \leqslant i \leqslant n$, such that $\varphi_1 \equiv \psi_1, ..., \varphi_n \equiv \psi_n$, then $\gamma(\varphi_1, ...,\varphi_n) \equiv \gamma(\psi_1, ..., \psi_n)$, for any formula $\gamma(a_1, ..., a_n)$.
\end{description}
\label{propriedades_relacao_consistencia_A_e_neg_A_e_preservacao_equiv_de_o_Cie_Cbr}
\end{properties}

Observing Nmatrix $\mathcal{M}_{Cbr}$, it is clear that \textbf{Cbr} satisfies (i) and (ii), but it is not \textit{self-extensional}, i.e., does not satisfies (iii). In \cite{CAR_CON_FUEN_REPLAC}, the authors present a class of self-extensional \textbf{LFI}s---denoted \textbf{R}$\mathbb{L}$, where $\mathbb{L}$ is some \textbf{LFI}---, the weakest of which is \textbf{RmbC}\footnote{The authors consider several self-extensional \textbf{LFI}s and their respective \textit{BALFI} models, but the results can be generalized. In \cite{TESTA_CON_MARTIN_CBR}, the logics \textbf{RCie} and \textbf{RCbr}---which are the self-extensional versions of \textbf{Cie} and \textbf{Cbr}, respectively---are considered. For the purposes of this document, only \textbf{RCbr} will be considered, and thus the following definitions and results will be adapted accordingly.}. Replacement satisfiability is obtained, briefly, as follows. First, two new \textit{global} inference rules---that is, rules valid only for theorems---and a new derivation notion are introduced\footnote{These inferential rules and the derivation notion are defined similarly to the necessitation rule in modal logic. They should be read as: if $\varphi \leftrightarrow \psi$ is a theorem, then $\# \varphi \leftrightarrow \# \psi$ is a theorem, for $\# \in \{\neg, \circ\}$. A more detailed version of the derivation definition can be found in \cite[p.6]{CAR_CON_FUEN_REPLAC}. However, a shorter version, sufficient for our purposes, is available in \cite[p.8]{TESTA_CON_MARTIN_CBR}.}:

\begin{table}[h!]
    \centering
    \begin{tabular}{c l p{2cm} c l}
         $\dfrac{\varphi \leftrightarrow \psi}{\neg \varphi \leftrightarrow \neg \psi}$ & $(R_\neg)$ & & $\dfrac{\varphi \leftrightarrow \psi}{\circ \varphi \leftrightarrow \circ \psi}$ & $(R_\circ)$  \\
    \end{tabular}
\end{table}
\FloatBarrier

\begin{definition}\label{definicao_derivacao_RmbC}
    We say that $\varphi$ is derivable in \textbf{R}$\mathbb{L}$, written $\vdash_{\text{R}\mathbb{L}} \varphi$, if there is a derivation in \textbf{R}$\mathbb{L}$ in the usual sense. On the other hand, $\varphi$ is derivable in \textbf{R}$\mathbb{L}$ from a set of premises $\Gamma$, written $\Gamma \vdash_{\text{R}\mathbb{L}} \varphi$, is either $\vdash_{\text{R}\mathbb{L}}  \varphi$ or there exists $\varphi_1, ...,\varphi_n \in \Gamma$, such that $\vdash_{\text{R}\mathbb{L}} (\varphi_1 \land ... \land \varphi_n) \to \varphi$.  
\end{definition}

The algebraization of \textbf{R}$\mathbb{L}$ in the Lindenbaum-Tarski sense consists of expansions of Boolean algebras through the addition of operators $\tilde{\neg}$ and $\tilde{\circ}$. The definition of \textit{BALFI} (Boolean Algebra with \textbf{LFI} Operators) for the class \textbf{R}$\mathbb{L}$, with $\mathbb{L} =$  \{\textbf{mbC}, \textbf{mbCciw}, \textbf{Cbr}\}, and the corresponding semantic logical consequence is given as follows\footnote{We denote the algebraic operations by $\sqcap$, $\sqcup$, $-$, and $\Rightarrow$ for \textit{meet}, \textit{join}, \textit{complement}, and \textit{implication}, respectively. As is usual in logical algebraization, valuations are expressed via homomorphisms $v: \mathcal{L}_{\Sigma_\circ} \to \mathfrak{B}$ as follows: $v(\#\varphi) = \tilde{\#} \; v(\varphi)$ for $\# \in \{\neg, \circ\}$, $v(\sim \! \varphi) = - v(\varphi)$, and $v(\varphi \# \psi) = v(\varphi) \; \tilde{\#} \; v(\psi)$ for \mbox{$\# \in \{\lor, \land, \to\}$} and $\tilde{\#} \in \{\sqcup, \sqcap, \Rightarrow\}$, respectively.}:

\begin{definition}[BALFI]
A Boolean algebra with \textbf{LFI} operators (\textit{BALFI}) is an algebra $\mathfrak{B} = \langle A, \sqcap, \sqcup, -, \Rightarrow, \tilde{\neg}, \tilde{\circ}, 0, 1 \rangle$, obtained by expanding a Boolean algebra $\mathfrak{A} = \langle A, \sqcap, \sqcup, -, \Rightarrow, 0, 1 \rangle$ with the unary operators $\tilde{\neg}, \tilde{\circ}$. For all $x \in A$: a \textit{BALFI} for \textbf{RmbC} is an algebra $\mathfrak{B}$ such that \mbox{$x \sqcup \tilde{\neg x} = 1$} and $x \sqcap \tilde{\neg} x \sqcap \tilde{\circ} x = 0$; A \textit{BALFI} for \textbf{RmbCciw} is a \textit{BALFI} for \textbf{RmbC} such that $\tilde{\circ} x = - (x \sqcap \tilde{\neg} x)$. A \textit{BALFI} for \textbf{RCbr} is a \textit{BALFI} for \textbf{RmbCciw} such that $\tilde{\neg} \tilde{\neg} x = x$. Given $\mathbb{L} =$ {\textbf{mbC}, \textbf{mbCciw}, \textbf{Cbr}}, the class of BALFIs for \textbf{R}$\mathbb{L}$ is denoted $\mathbb{B}(\text{R}\mathbb{L})$.
\label{definicao_balfis_RLs}
\end{definition}

\begin{definition}[Logical consequence in \textit{BALFI}s]
Let $\mathfrak{B} \in \mathbb{B}(\text{R}\mathbb{L})$. (i) We say that $\varphi$ is valid in $\mathfrak{B}$, that is, $\models_{\mathfrak{B}} \varphi$, if $v(\varphi) = 1$ for every homomorphism $v: \mathcal{L}_{\Sigma_\circ} \to \mathfrak{B}$; (ii) A formula $\varphi$ is valid in $\mathbb{B}(\text{R}\mathbb{L})$, that is, $\models_{\mathbb{B}(\text{R}\mathbb{L})}$, if it is valid in every $\mathfrak{B} \in \mathbb{B}(\text{R}\mathbb{L})$; (iii) For every $\Gamma \cup \{\varphi\} \subseteq \mathcal{L}_{\Sigma_\circ}$, we say that $\varphi$ is a consequence of $\Gamma$ in $\mathbb{B}(\text{R}\mathbb{L})$, that is, $\Gamma \models_{\mathbb{B}(\text{R}\mathbb{L})} \varphi$, if either $\varphi$ is valid in $\mathbb{B}(\text{R}\mathbb{L})$ or there exists $\varphi_1, ..., \varphi_n \in \Gamma$, such that $(\varphi_1 \land...\land \varphi_n) \to \varphi$ is valid in $\mathbb{B}(\text{R}\mathbb{L})$.
\label{definicao_consequencia_logica_BALFI}
\end{definition}

The soundness and completeness theorem---that is, for every $\Gamma \cup \{\varphi\} \subseteq \mathcal{L}_{\Sigma_\circ}$, $\Gamma \vdash_{R\mathbb{L}} \varphi$ if and only if $\Gamma \models_{\mathbb{B}(\text{R}\mathbb{L})} \varphi$---is obtained\footnote{See \cite[p.8]{CAR_CON_FUEN_REPLAC} and \cite[p.9]{TESTA_CON_MARTIN_CBR}.}. Since in logic \textbf{Cbr} the property \ref{propriedades_relacao_consistencia_A_e_neg_A_e_preservacao_equiv_de_o_Cie_Cbr} (ii) holds, the inference rule $(R_\circ)$ is superfluous; that is, it can be derived in the self-extensional logic \textbf{RCbr}. Therefore, \textbf{RCbr} = \textbf{Cbr} $+ (R_\neg)$. In \cite[p.9-18]{TESTA_CON_MARTIN_CBR}, the authors present an interesting algebraic model $\mathfrak{B} = \mathbb{B}(\text{\textbf{RCbr}})$ that validates all its axioms while remaining a paraconsistent \textbf{LFI}. Furthermore---and importantly---$\mathfrak{B}$ is a countermodel for $\not\models_{RCbr} \circ \!\circ\! \varphi$, $\not\models_{RCbr} \circ \varphi \to \circ \!\circ\! \varphi$, $\not\models_{RCbr} (\varphi \land \circ\varphi) \to \circ \!\circ\! \varphi$, and $\not\models_{RCbr} (\neg \varphi\land \circ \varphi) \to \circ \!\circ\! \varphi$, which makes the \textbf{RCbr} logic ideal for belief revision purposes\footnote{Without these countermodel, the \textbf{AGM}$\circ$ system would be unable to overturn beliefs, since they would automatically be considered strongly accepted. See \cite{Testetal} for more specific information about \textbf{AGM}$\circ$ system and its epistemic attitudes.}.

\section{AGM\texorpdfstring{$p_{abd}$}{} system preliminaries}
\label{agmpabd_preliminaries}

Thanks to the special properties of the logic \textbf{RCbr}, the \textbf{AGM}$p_{abd}$ system was developed to allow both \textbf{CPL} and \textbf{RCbr} as underlying logics, without substantial changes to the statements, definitions and proofs\footnote{This is possible mainly due to the following reasons: the guarantee of substitutions between formulas, the satisfiability of classical properties, and the non-direct use of the paraconsistent connectives $\neg$ and $\circ$ in its postulates and constructions. As we have seen, the logic \textbf{RCbr}, as an extension of \textbf{mbC}, is a standard Tarskian logic (definition \ref{definicao_logica_tarskiana_standard}) and supraclassical (definition \ref{definicao_logica_tarskiana_supraclassuca}). Thus, it satisfies both structural substitutions and fundamental rules and properties of \textbf{CPL}, such as \textit{modus ponens}, the \textit{deduction theorem}, \textit{disjunction of premises} and \textit{proof by cases} (see properties \ref{propriedades_teorema_da_deducao_disjuncao_das_premissas}). Moreover, the logic \textbf{RCbr} is self-extensional, i.e., it satisfies the \textit{replacement} property, which naturally allows substitutions between logically equivalent formulas---a particularly important feature for the \textit{extensionality} postulate and its developments, as we shall see. Furthermore, substitutions at the metalogical level, according to well-known conventions in set theory and commonly used in manipulating the properties of the Tarskian consequence operator $Cn$, are evidently natural.}. This quite special feature makes paraconsistency optional, should it be convenient to overcome problematic scenarios in the classical world, as we shall see below. With this in mind, for reasons of simplification, the notation $\mathbb{L} \in$ \{\textbf{CPL}, \textbf{RCbr}\} and the subscripts $\vdash_{\mathbb{L}}$, $\nvdash_{\mathbb{L}}$, $\equiv_{\mathbb{L}}$ and $Cn_{\mathbb{L}}$, used many times throughout this article, should be considered a metatheoretical indicator, such that each statement, definition, theorem and proof, when pertaining to both logics, is presented only once, with $\mathbb{L}$ being a parameter that can be instantiated as either \textbf{CPL} or \textbf{RCbr}. In both cases, we have that $\mathbb{L} = \langle \mathcal{L}_{\Sigma_{*}}, Cn_{\mathbb{L}} \rangle$ is a standard and supraclassical Tarskian logic, where, on the one hand, in the case of $\mathbb{L} =$ \{\textbf{CPL}\}, we should consider the language $\mathcal{L}_{\Sigma_*}$ as generated by the signature $\Sigma_* = \{\land, \lor, \to, \til, \bot, \top \}$, where $\til$ is the primitive classical negation. On the other hand, when $\mathbb{L} =$ \{\textbf{RCbr}\}, we should consider the language $\mathcal{L}_{\Sigma_*}$ as generated by the signature $\Sigma_* = \{\land, \lor, \to, \neg, \circ \}$---the same used in the exposition of \textbf{mbC}s in Section \ref{rcbr_logic}. In this case, the \textit{falsum} particle and the strong negation will be defined by \mbox{$\bot_\psi =_{\text{def}} \psi \land \neg \psi \land \circ \psi$} and (ii) $\til_\psi \varphi =_{\text{def}} (\varphi \to \bot_\psi)$.

At this point, I present some important definitions for the following sections:

\begin{definition}  \label{conjunto_th_de_todos_as_teorias_agmpabd}
    Let $\mathbb{L} \in \{\textbf{CPL}, \textbf{RCbr}\}$. The set $Th(\mathbb{L})$ of all theories of the logic $\mathbb{L}$ is given by:
    
    \centering
    $Th(\mathbb{L})=\{\Theta \subseteq \mathcal{L}_{\Sigma_*} \ : \ \Theta=Cn_{\mathbb{L}}(\Theta)\}.$
\end{definition}

\begin{definition}[Belief set]
     A set $\Theta \subseteq \mathcal{L}_{\Sigma_*}$ of sentences is a belief set if, and only if, $\Theta \in Th(\mathbb{L})$.
\label{definicao_conj_crencas_agmpabd}
\end{definition}

Since $\Theta$ is a belief set, the (classical) triggers presented in Section \ref{taxogeral} can be defined as follows:

\begin{definition}[Classical abductive process triggers] \label{gatilhos_processo_abdutivo_classico_agmpabd}
Let $\mathbb{L} = \{\textbf{CPL}\}$, $\Theta \cup \{\varphi\} \subseteq \mathcal{L}_{\Sigma_{*}}$ be a set of formulas and $\Theta \in Th(\mathbb{L})$. $\varphi$ is a trigger of the abductive process with respect to $\Theta$ if:
\end{definition}
    
\begin{description}
\leftskip 40pt
    \item (i) \textit{Abductive novelty}: $\varphi \notin \Theta$ and $\til \varphi \notin \Theta$;
    \item (ii) \textit{Abductive anomaly}: $\varphi \notin \Theta$ and $\til \varphi \in \Theta$.
\end{description}

Classically, the abductive process, with its triggers and its operations $\expfip$ and $\revfipi$ of abductive expansion and (internal) revision \textbf{AGM}$p_{abd}$, respectively, as well as other components of the taxonomy, can be better detailed visually according to Figure \ref{figura_processo_abdutivo_classico_agmpabd}. Let us recall, however, that only the abductive expansion operation $\expfip$ is the focus of this paper:

\begin{figure}[h!]
\begin{center}
\includegraphics[width=\textwidth, keepaspectratio=true]{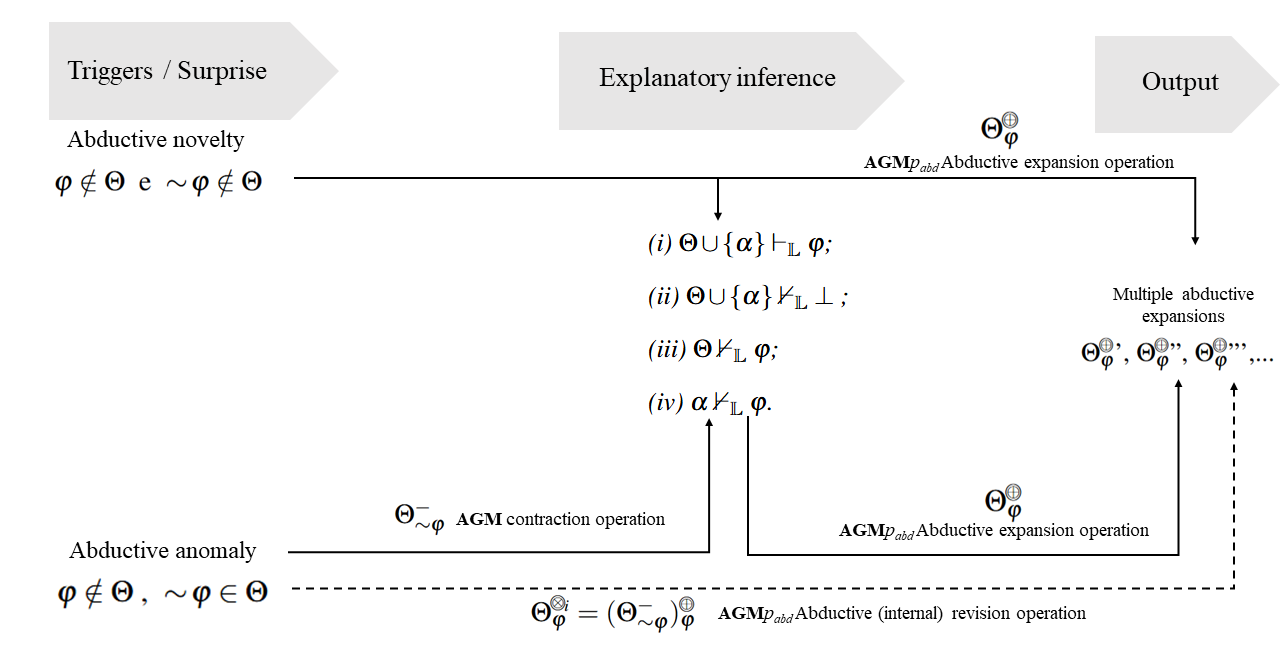}
\caption{The classical abductive process and the AGM$p_{abd}$ operations.}
\label{figura_processo_abdutivo_classico_agmpabd}
\end{center}
\end{figure}
\FloatBarrier

In order to make the exposition and the formal proofs ahead more simplified, as we shall see, the relation between each of these triggers and their respective operations is established by means of \textit{appropriate pairs}. Thus, in the classical case, we have the following definitions.

\begin{definition}[Appropriate pair for classical \textbf{AGM}$p_{abd}$ abductive expansion] \label{definição_par_apropriado_expansao_classica_agmpabd}
Let $\mathbb{L} = \{\textbf{CPL}\}$, $\Theta \cup \{\varphi\} \subseteq \mathcal{L}_{\Sigma_*}$ be a set of formulas and $\Theta \in Th(\mathbb{L})$. We say that $(\Theta, \varphi)$---and its substitutions---is an appropriate pair for classical \textbf{AGM}$p_{abd}$ abductive expansion---notation $\papfi$---if $\varphi$ is an abductive novelty trigger (definition \ref{gatilhos_processo_abdutivo_classico_agmpabd} (i)), i.e., if $\varphi \notin \Theta$ and $\til \varphi \notin \Theta$.
\end{definition}

Let us now consider the paraconsistent context. When $\mathbb{L} = \{\textbf{RCbr}\}$, the richness of the language allows us to consider, combinatorially, for $\varphi \notin \Theta$, four triggers, according to the following definition.

\begin{definition}[Paraconsistent abductive process triggers] \label{gatilhos_paraconsistentes_agmpabd}
Let $\mathbb{L} = \{\textbf{RCbr}\}$, $\Theta \cup \{\varphi\} \subseteq \mathcal{L}_{\Sigma_*}$ be a set of formulas and $\Theta \in Th(\mathbb{L})$. $\varphi$ is a trigger of the abductive process with respect to $\Theta$ if:
\end{definition}

\begin{description}
\leftskip 40pt
    \item (i) \textit{Strong abductive novelty}: $\varphi \notin \Theta$, $\circ \varphi \notin \Theta$ and $\neg \varphi \notin \Theta$;
    \item (ii) \textit{Weak abductive novelty}: $\varphi \notin \Theta$, $\circ \varphi \in \Theta$ and $\neg \varphi \notin \Theta$;
     \item (iii) \textit{Weak abductive anomaly}: $\varphi \notin \Theta$, $\circ \varphi \notin \Theta$ and $\neg \varphi \in \Theta$;
     \item (iv) \textit{Strong abductive anomaly}: $\varphi \notin \Theta$, $\circ \varphi \in \Theta$ and $\neg \varphi \in \Theta$.  
\end{description}

Thus, analogously to the classical perspective, Figure \ref{figura_processo_abdutivo_paraconsistente_agmpabd} details the abductive process in light of the new triggers and their respective \textbf{AGM}$p_{abd}$ operations.

\begin{figure}[h!]
\begin{center}
\includegraphics[width=\textwidth, keepaspectratio=true]{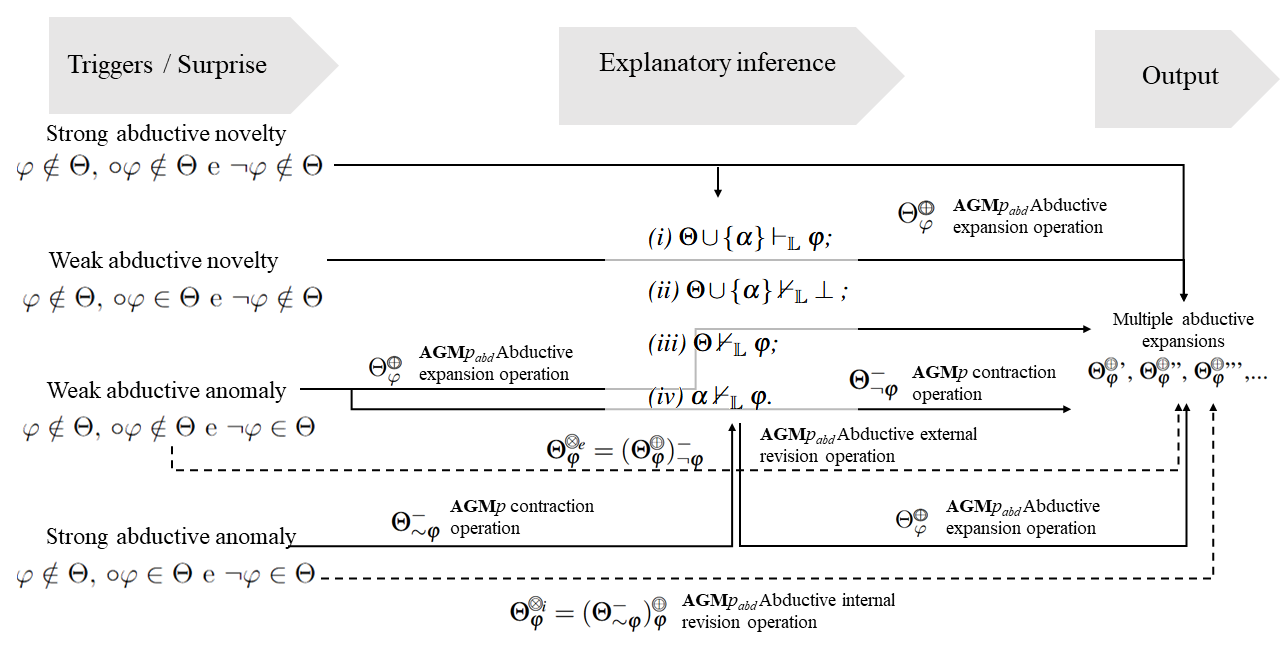}
\caption{The paraconsistent abductive process and the AGM$p_{abd}$ operations.}
\label{figura_processo_abdutivo_paraconsistente_agmpabd}
\end{center}
\end{figure}
\FloatBarrier

Unlike the classical case, therefore, the increased expressive power of the language of \textbf{RCbr}---and of any \textbf{LFI}, in fact---allows us to consider four more sophisticated and interesting interpretations of the initial surprise when faced with a surprising fact. The definitions of \textit{appropriate pairs} for the respective operations, in the paraconsistent case, introduce new elements and require some care. Note that the first three triggers---strong and weak abductive novelties and weak abductive anomaly---can initiate an abductive expansion operation $\expfip$. So, we may define:

\begin{definition}[Appropriate pair for paraconsistent \textbf{AGM}$p_{abd}$ abductive expansion] \label{definição_par_apropriado_expansao_paraconsistente_agmpabd}
Let $\mathbb{L} = \{\textbf{RCbr}\}$, $\Theta \cup \{\varphi\} \subseteq \mathcal{L}_{\Sigma_*}$ be a set of formulas and $\Theta \in Th(\mathbb{L})$. We say that $(\Theta, \varphi)$---and its substitutions---is an appropriate pair for paraconsistent abductive expansion---notation $\papfip$---if (i) $\varphi$ is a strong abductive novelty trigger for $\Theta$ or (ii) $\varphi$ is a weak abductive novelty trigger for $\Theta$ or (iii) $\varphi$ is a weak abductive anomaly trigger for $\Theta$ (definition \ref{gatilhos_paraconsistentes_agmpabd}).
\end{definition}

However, since the logic \textbf{RCbr} is an extension of \textbf{mbCciw}, as we have seen, and $\til\varphi \equiv_{\textbf{mbCciw}} \circ \varphi \land \neg \varphi$ (footnote \ref{rodape_mbcciw_properties}), as well as $\Theta \in Th(\mathbb{L})$, then the following observation conveniently shows that, when using strong negation, the strong abductive novelty, weak abductive novelty and weak abductive anomaly triggers in the \textbf{AGM}$p_{abd}$ system ``collapse'' into the classical case\footnote{The differentiated treatment for each of these cases, within the paraconsistent abductive expansion operation, is precisely the proposal of the \textbf{AGM}$\circ_{abd}$ system to be presented in the second paper.}.

\begin{remark}\label{remark_par_apropriado_expansao_negacao_forte_agmpabd}
It is the case that $\papfip$ (definition \ref{definição_par_apropriado_expansao_paraconsistente_agmpabd}) if and only if $\varphi \notin \Theta$ and $\til\varphi \notin \Theta$.
\end{remark}

This means that $\papfi$ denotes an appropriate pair for abductive expansion in the classical case---that is, when $\mathbb{L} = \{\textbf{CPL}\}$ is the underlying logic---and $\papfip$ denotes an appropriate pair for paraconsistent abductive expansion---namely, when $\mathbb{L} = \{\textbf{RCbr}\}$ is the underlying logic. However, since both are defined by (i) $\varphi \notin \Theta$ and (ii) $\til\varphi \notin \Theta$, they can be used interchangeably in statements and formal proofs. By convention and for notational simplification, therefore, the following more general definition will be adopted.

\begin{definition}[Appropriate pair for \textbf{AGM}$p_{abd}$ abductive expansion]
Let $\mathbb{L} \in \{\textbf{CPL}, \textbf{RCbr}\}$, $\Theta \cup \{\varphi\} \subseteq \mathcal{L}_{\Sigma_{*}}$ be a set of formulas and $\Theta \in Th(\mathbb{L})$. $(\Theta, \varphi)$ is an appropriate pair for \textbf{AGM}$p_{abd}$ abductive expansion (classical or paraconsistent)---notation $\papfig$---if (i) $\varphi \notin \Theta$ and (ii) $\til \varphi \notin \Theta$.
\label{definicao_par_apropriado_expansao_LPC_RCbr_agmpabd}
\end{definition}
\section{AGM\texorpdfstring{$p_{abd}$}{} abductive expansion  - postulates}
\label{agmpabd_expansion_postulates}

Before presenting the postulates of the operation, as a first step, I propose to formalize the following taxonomic component, the adequate explanatory inference, described in Section \ref{taxogeral}, accordingly to the rationality criterion of the \textit{explanatory adequacy}, with its respective requirements, as follows:

\begin{definition} \label{conj_hip_abdutivas_agmpabd}
Let $\mathbb{L} \in \{\textbf{CPL}, \textbf{RCbr}\}$, $\Theta \cup \{\alpha\} \cup \{\varphi\} \subseteq \mathcal{L}_{\Sigma_*}$ and $\varphi \notin \Theta$. The set $H(\Theta, \varphi)$, given by: $$H(\Theta,\varphi)=\{\alpha \in \mathcal{L}_{\Sigma_*} \ : (i) \ \Theta \cup \{\alpha\} \vdash_{\mathbb{L}} \varphi, \ (ii) \ \Theta \cup \{\alpha\} \nvdash_{\mathbb{L}} \bot \ \mbox{ and } \ (iii) \ \{\alpha\} \nvdash_{\mathbb{L}} \varphi\}$$
is the set of \textit{abductive hypotheses} for $(\Theta,\varphi)$.
\end{definition}

The following (desirable) results, in line with the rationality criterion of the \textit{surprising fact}, can be easily demonstrated.

\begin{lemma} [Partial self-explanation]  \label{lema_sempre_existe_auto_explicacao_agmpabd}
    If $\papfig$ is the case (definition \ref{definicao_par_apropriado_expansao_LPC_RCbr_agmpabd}), then there always exists some $\alpha \in \mathcal{L}_{\Sigma_*}$ that satisfies criteria (i)-(iii) of definition \ref{conj_hip_abdutivas_agmpabd}. Hence, $H(\Theta, \varphi) \neq \emptyset$.
\end{lemma}

\begin{remark} \label{alpha_notin_theta_agmpabd}
    Let $H(\Theta, \varphi)$ be the set of abductive hypotheses for the pair $(\Theta,\varphi)$ (definition \ref{conj_hip_abdutivas_agmpabd}). It is the case that \mbox{$\Theta \cap H(\Theta,\varphi)=\emptyset$}. 
\end{remark}

Thanks to the logic \textbf{RCbr}, the following result can be obtained, whether in the classical or the paraconsistent case:

\begin{lemma}  \label{conj_hipoteses_abdutivas_iguais_varphi1_varphi2_agmpabd}
    Let $\papfig$ and $\pappsig$. Let $H(\Theta, \varphi)$ and $H(\Theta, \psi)$ be the sets of abductive hypotheses for $\varphi$ and $\psi$, respectively. It is the case that if $\vdash_{\mathbb{L}} \varphi \leftrightarrow \psi$, then $H(\Theta, \varphi) = H(\Theta, \psi)$.    
\end{lemma}

Having presented the preliminary formal notions of the system, the postulates that characterize the \textbf{AGM}$p_{abd}$ abductive expansion operation are given below.

\begin{definition}[Postulates for \textbf{AGM}$p_{abd}$ abductive expansion] \label{definicao_postulados_expansao_abdutiva_agmpabd}
Let $\papfig$ (definition \ref{definicao_par_apropriado_expansao_LPC_RCbr_agmpabd}) and $H(\Theta, \varphi)$ be the set of abductive hypotheses for $(\Theta, \varphi)$ (definition \ref{conj_hip_abdutivas_agmpabd}). An \textbf{AGM}$p_{abd}$ abductive expansion of $\Theta$ by $\varphi$---denoted $\expfip$---is a function $\;\ooplusg:Th(\mathbb{L}) \times \mathcal{L}_{\Sigma_*} \to Th(\mathbb{L})$ from pairs of belief sets and sentences of the language to belief sets that satisfies the following postulates:
\end{definition}

\begin{description}
    \item[$\postp1$] $\expfip$ is a belief set \hfill (\textit{closure})
    \item[$\postp2$] $\Theta \subset \expfip$ \hfill (\textit{inclusion})
    \item[$\postp3$] $\expfip \cap H(\Theta,\varphi) \neq \emptyset$ \hfill (\textit{success})
    \item[$\postp4$] $\expfip \nvdash_{\mathbb{L}}\bot$ \hfill (\textit{non-triviality})
    \item[$\postp5$] If $\vdash_{\mathbb{L}} \varphi \leftrightarrow \psi$, then $\expfip = \exppsip$ \hfill (\textit{extensionality})
    \item[$\postp6$] $\expfip \subseteq Cn_{\mathbb{L}}(\expfilorp \cup \{\varphi\})$ \hfill (\textit{supplementary 1})
    \item[$\postp7$] If $\til \varphi \notin \expfilorp$, then $\expfilorp \subseteq \expfip$ \hfill (\textit{supplementary 2})
\end{description}

The first feature to be noted is the requirement that $\papfig$ be an appropriate pair for abductive expansion---that is, $\varphi \notin \Theta$ and $\til \varphi \notin \Theta$ and their substitutions---applied to the definition \ref{definicao_postulados_expansao_abdutiva_agmpabd} of the operation itself, as a prerequisite to the postulates. This is very relevant, because, in accordance with the rationality criteria and with the proposed taxonomy, the function $\ooplusg$ underlying the operation $\expfip$ cannot even be invoked in contexts where it is notably the case that $\varphi \in \Theta$ or $\til \varphi \in \Theta$. For, in that case, we have an ill-defined function and, therefore, inappropriate for any purpose. Postulate $\postp1$ is well known and requires no further consideration. Postulate $\postp2$ tells us that $\Theta$ is a \textit{proper} subset of $\expfip$. In other words, the operation does not consider $\expfip = \Theta$ as a legitimate expansion, which is fully justified by the rationality criterion of the \textit{surprising fact}. Postulate $\postp3$ of success, in turn, guarantees that there is at least one abductive hypothesis $\alpha \in H(\Theta, \varphi)$ in the expanded belief set. Notably, this postulate authorizes multiple possible expansions, according to the number of hypotheses present in the set $H(\Theta, \varphi)$\footnote{It is important to note the many distinctions, up to this point, in relation to the postulates proposed by Pagnucco \cite[p.103-105]{Pagnucco}. First, the author allows $\varphi \in \Theta$ and, therefore, that $\expfi = \Theta$. As discussed earlier, this is a situation that I believe is not desirable, whether in the context of abduction or in the context of explanation. According to observation \ref{varphi_in_anel_agmpabd} below, as in Pagnucco's system, it is the case that $\varphi \in \expfip$. However, in the \textbf{AGM}$p_{abd}$ abductive expansion operation, $\varphi$, the phenomenon to be explained, is found exclusively in the ``ring'' $\expfip \setminus \Theta$, effectively characterizing an expansion of the initial epistemic state. Moreover, unlike definition \ref{conj_hip_abdutivas_agmpabd} of $H(\Theta, \varphi)$, Pagnucco's notion of explanation \cite[p.79]{Pagnucco} does not block $\{\alpha\} \vdash \varphi$, something which, as discussed in Section \ref{taxogeral}, I consider inappropriate. Finally, Pagnucco's \textit{limited success} postulate, despite being quite elegant, depends on a very specific condition being satisfied: the adopted language must be \textit{finite}---see definition of abductive expansion \cite[p.102]{Pagnucco} and theorem 5.2.1 \cite[p.105]{Pagnucco}---a requirement that will not be adopted here.}. With these three postulates and, evidently, considering $\papfig$, as the definition requires, we obtain the following observation:

\begin{remark} \label{varphi_in_anel_agmpabd}
   If $\expfip$ satisfies postulates $\postp1$ - $\postp3$ from definition \ref{definicao_postulados_expansao_abdutiva_agmpabd}, then $\varphi \in \expfip\setminus\Theta$. 
\end{remark}

In accordance, once again, with the rationality criterion of the \textit{surprising fact}---and unlike Pagnucco's system---a \textit{failure} postulate is not necessary, as the following results indicate. This is, therefore, an always successful operation, given $\papfig$.

\begin{lemma} \label{lema_alpha_ou_fi_in_hip_abdutivas_agmpabd}
    Let $\papfig$ and $H(\Theta, \varphi)$ be the set of abductive hypotheses for $(\Theta, \varphi)$. If $\alpha \in H(\Theta,\varphi)$, then it is the case that $\alpha \lor \varphi \in H(\Theta, \varphi)$.
\end{lemma}

\begin{lemma} \label{lema_alpha_ou_fi_in_expansao_agmpabd}
If $\expfip$ satisfies postulates $\postp1$ - $\postp3$ from definition \ref{definicao_postulados_expansao_abdutiva_agmpabd}, then, for $\alpha \in H(\Theta, \varphi)$, it is the case that $\alpha \lor \varphi \in \expfip$.
\end{lemma}

\begin{corollary} \label{corolario_alpha_ou_fi_sempre_expande_agmpabd}
    On the one hand, if $\expfip$ satisfies postulates $\postp1$ - $\postp3$ from definition \ref{definicao_postulados_expansao_abdutiva_agmpabd}, then, by lemma \ref{lema_alpha_ou_fi_in_expansao_agmpabd}, for $\alpha \in H(\Theta, \varphi)$, it is the case that $\alpha \lor \varphi \in \expfip$. On the other hand, by lemma \ref{lema_alpha_ou_fi_in_hip_abdutivas_agmpabd}, if $\papfig$ is the case---a requirement of definition \ref{definicao_postulados_expansao_abdutiva_agmpabd} of the operation $\expfip$---then it is also the case that $\alpha \lor \varphi \in H(\Theta, \varphi)$. Thus, it is always the case that $\alpha \lor \varphi \in \expfip \cap H(\Theta, \varphi)$. Hence, these conditions are sufficient for a failure postulate not to be necessary.
\end{corollary}

Postulate $\postp4$ evidently expresses a distinct behavior of the operation depending on the logic underlying the \textbf{AGM}$p_{abd}$ system. If $\mathbb{L} = \{\textbf{CPL}\}$, the \textbf{AGM}$p_{abd}$ abductive expansion operation prevents expansion to contradiction. On the other hand, in the case where $\mathbb{L} = \{\textbf{RCbr}\}$, the postulate blocks expansion to triviality, while allowing contradictions---hence its paraconsistent feature. Moreover, since $\papfig$ is an initial requirement, $\Theta \neq \Theta_\bot$, which satisfies the rationality criterion of \textit{non-triviality}, i.e., there is no room for triviality from the beginning to the end of the abductive expansion process. Postulate $\postp5$ guarantees the \textit{Principle of Irrelevance of Syntax}. Note that this postulate is a direct consequence of lemma \ref{conj_hipoteses_abdutivas_iguais_varphi1_varphi2_agmpabd}---that is, if $\vdash_{\mathbb{L}} \varphi \leftrightarrow \psi$, then $H(\Theta, \varphi) = H(\Theta, \psi)$. Classically, this postulate is obvious\footnote{It is quite interesting to note that, in his system, Pagnucco presents the \textit{strong extensionality} postulate---if $\Theta \vdash \varphi \leftrightarrow \psi$, then $\expfi = \exppsi$. In the \textbf{AGM}$p_{abd}$ system, however, the strong extensionality postulate cannot be considered, whether in the classical or the paraconsistent case, due to observation \ref{remark_theta_varphi1_varphi2_nao_e_o_caso_agmpabd}. In its brief proof (found in Appendix \ref{apendiceB}), it becomes evident that if $\Theta \vdash_{\mathbb{L}} \varphi \leftrightarrow \psi$, but $\nvdash_{\mathbb{L}} \varphi \leftrightarrow \psi$, then $\psi \in H(\Theta, \varphi)$ and $\varphi \in H(\Theta, \psi)$, but $\psi \notin H(\Theta, \psi)$ and $\varphi \notin H(\Theta, \varphi)$. This means that, if we opted for the strong extensionality postulate instead of $\postp5$, the consequent $\expfip = \exppsip$ would make the postulate ill-defined. Note that if $\Theta \vdash_{\mathbb{L}} \varphi \leftrightarrow \psi$, but $\nvdash_{\mathbb{L}} \varphi \leftrightarrow \psi$, in an \textbf{AGM}$p_{abd}$ abductive expansion operation $\expfip$, it may be the case, by the success postulate $\postp3$, that $\psi \in \expfip \cap H(\Theta, \varphi)$. However, certainly $\psi \notin \exppsip \cap H(\Theta, \psi)$. This situation does not occur if the requirement is the weak extensionality postulate $\postp5$ (note that if $\vdash_{\mathbb{L}} \varphi \leftrightarrow \psi$, then, by lemma \ref{conj_hipoteses_abdutivas_iguais_varphi1_varphi2_agmpabd}, $H(\Theta, \varphi)=H(\Theta, \psi)$, since $\{\varphi, \psi\} \notin H(\Theta, \varphi)$ and, similarly, $\{\varphi, \psi\} \notin H(\Theta, \psi)$). To a large extent, this is a feature imposed by criterion (iii) of definition \ref{conj_hip_abdutivas_agmpabd} of the set of abductive hypotheses, absent in Pagnucco's system.}, but in the paraconsistent context, this postulate is only possible because the logic \textbf{RCbr} satisfies the \textit{replacement} property. Finally, the supplementary postulates $\postp6$ and $\postp7$, as well as their properties\footnote{See properties \ref{proposition_propriedades_teta6_agmpabd} and \ref{proposition_propriedades_teta7_agmpabd} in Appendix \ref{apendiceB}. The proofs are, with small adjustments, the same as those obtained by Pagnucco.}, are, strictly and curiously, the same as in Pagnucco's system\footnote{It should be borne in mind, however, that, unlike that system, given definition \ref{definicao_postulados_expansao_abdutiva_agmpabd} of the postulates, the iterated abductive expansion operation $\expfilorp$ requires, by substitution, $\paplorg$. Therefore, when the operation $\expfilorp$ is invoked, we have $\varphi \lor \psi \notin \Theta$ and $\til(\varphi \lor \psi) \notin \Theta$. Considering that, for $\mathbb{L} \in \{\textbf{CPL}, \textbf{RCbr} \}$, De Morgan's Laws hold for classical (strong) negation $\til$ and that $\Theta \in Th(\mathbb{L})$, then we evidently have $\til \varphi \land \til \psi \notin \Theta$.}.

\subsection{Distinctions between classical and paraconsistent \textbf{AGM}\texorpdfstring{$p_{abd}$}{} abductive expansions} \mbox{}
\label{distinctions_classic_paraconsistent}

At this point, it is necessary to establish some important behavioral distinctions in the abductive expansion operation $\expfip$, for the cases $\mathbb{L} = \{\textbf{CPL}\}$ and $\mathbb{L} = \{\textbf{RCbr}\}$, without, however, compromising the formal proofs. Similarly to the interesting examples brought in the article \cite[p.325-326]{Bueno_Car_Con_Ab_2017} (and also revisited in \cite[p.25-27]{Bueno_Car_Con_Ant_LET_2022})\footnote{It is important to highlight some important distinctions between these works and the system developed in this paper. First, the authors, in both works, focus on paraconsistent abductive tableaux, not on \textbf{AGM}-style abductive operations, with their respective postulates and constructions. Similarly, the set of initial theories or previous beliefs in those works is not closed under logical consequences, as is the case of the \textbf{AGM}$p_{abd}$ system. Moreover, there are subtle but important differences between the notion of explanation adopted in those works (see, for example, \cite[p.24-25]{Bueno_Car_Con_Ant_LET_2022}) and the notion adopted in this paper (Subsection \ref{taxogeral}), which I will not go into here. No less important, the paraconsistent logics used for the construction of their abductive tableaux are \textbf{mbC}, in the case of the first article, and \textbf{LET}$_k$, in the case of the second. The examples presented, however, with the appropriate modifications, serve perfectly for the purposes of this exposition.}, the abductive expansion operation $\expfip$ behaves, in some cases, identically in the classical and paraconsistent cases---example (i) below---but, in other cases, quite differently---examples (ii)-(iv) below. After all, a richer language, in this case the language of \textbf{RCbr}, also enables the representation of richer epistemic changes and explanations.

Let us first consider (i) cases where the result of the operation $\expfip$ is the same for the classical and paraconsistent cases. Let $\papfig$ and $\{\psi \to \varphi, \delta \to \varphi\} \subseteq \Theta$. Naturally, $\psi,\delta \in H(\Theta, \varphi)$---by conditions (i)-(iii) of definition \ref{conj_hip_abdutivas_agmpabd}. By postulates $\postp1$ - $\postp4$ and observation \ref{varphi_in_anel_agmpabd}, we have that $\{\psi \to \varphi, \delta \to \varphi, \psi, \varphi\} \subseteq \expfip$, $\{\psi \to \varphi, \delta \to \varphi, \delta, \varphi\} \subseteq \expfip$ and $\{\psi \to \varphi, \delta \to \varphi, \psi, \delta, \varphi\} \subseteq \expfip$ (as well as their logical consequences) are all possible abductive expansions, both in the classical and in the paraconsistent cases. The same occurs for $\{\psi \to \delta, \delta \to \varphi\} \subseteq \Theta$. The indistinction between both expansions is clearly linked to the fact that there are no paraconsistent operators $\neg$ and $\circ$ involved in the operations.

On the other hand, let us consider three cases where the adoption of paraconsistency alters the behavior of the operation. Consider (ii) $\papfip$ (definition \ref{definicao_par_apropriado_expansao_LPC_RCbr_agmpabd}) and $\{\psi \to \varphi, \neg \psi\} \subseteq \Theta$. Since, in this case, $\mathbb{L} = \{\textbf{RCbr}\}$, we have that $\psi \in H(\Theta, \varphi)$---that is, unlike the classical case, $\psi$ becomes a possible abductive hypothesis for explaining $\varphi$, according to conditions (i)-(iii) of definition \ref{conj_hip_abdutivas_agmpabd}---and, notably, by postulates $\postp1$ - $\postp4$ and observation \ref{varphi_in_anel_agmpabd}, we have that $\{\psi \to \varphi, \neg \psi, \psi, \varphi \} \subseteq \expfip$ becomes a possible paraconsistent abductive expansion, insofar as $\expfip$ comes to accept contradictory hypotheses\footnote{This example is called by the authors, in \cite[p.325]{Bueno_Car_Con_Ab_2017} and \cite[p.26]{Bueno_Car_Con_Ant_LET_2022}, as ``Impossible explanations explained.'' Indeed, given $\Theta$, we have a new explanation $\psi$ for the phenomenon $\varphi$ that would be impossible in the classical case, since, if $\mathbb{L} = \{\textbf{CPL}\}$, then $\expfip = \Theta_\bot$.}. In the classical case, this expansion would evidently be blocked by postulate $\postp4$. Note, moreover, that by postulate $\postp1$, $\expfip \in Th(\mathbb{L})$, so, since $\psi \land \neg \psi \vdash_{\textbf{mbCciw}} \neg \circ \psi$\footnote{In fact, this derivation is already valid in \textbf{mbC}.}, and \textbf{RCbr} is an extension of \textbf{mbCciw}, we have that $\neg \circ \psi \in \expfip$. In other words, $\psi$ becomes an abductive hypothesis for $\varphi$, at the cost of the agent inevitably assuming its inconsistency.

Now let us consider the following case (iii) from the classical point of view, initially. Let $\papfi$ (definition \ref{definição_par_apropriado_expansao_classica_agmpabd}) and $\{\varphi \to \psi, \varphi \to \til \psi\} \subseteq \Theta$. In this case, since $\Theta \in Th(\mathbb{L})$, because $(\Theta, \varphi)$ is an appropriate pair for abductive expansion and, by \textit{reductio}, $\{\varphi \to \psi, \varphi \to \til \psi\} \vdash_{\textbf{CPL}} \til \varphi$, then the agent would be forced to (strongly) negate $\varphi$, i.e., $\til \varphi \in \Theta$. But this (strongly) contradicts the initial assumption that $\papfi$. Hence, the belief set $\Theta$ would be rejected from the start by the operation. Even if the expansion operation $\expfip$ were authorized, since $\til \varphi \in \Theta$, given postulates $\postp1$ - $\postp3$, and, by observation \ref{varphi_in_anel_agmpabd}, $\varphi \in \expfip$, then $\expfip \vdash \bot$, which violates postulate $\postp4$. On the other hand, let us reformulate the example for the paraconsistent case. Let $\papfip$ (definition \ref{definicao_par_apropriado_expansao_LPC_RCbr_agmpabd}) and $\{\varphi \to \psi, \varphi \to \neg \psi\} \subseteq \Theta$. Since $\{\varphi \to \psi, \varphi \to \neg \psi\} \nvdash_{\textbf{RCbr}} \neg \varphi$ (and also $\{\varphi \to \psi, \varphi \to \neg \psi\} \nvdash_{\textbf{RCbr}} \til \varphi$), we have an initial belief set that is not blocked initially by the operation. In fact, given postulates $\postp1$ - $\postp4$ and observation \ref{varphi_in_anel_agmpabd}, $\{\varphi \to \psi, \varphi \to \neg \psi, \varphi\} \subseteq \expfip$ becomes a perfectly admissible operation\footnote{Recalling that, according to lemma \ref{lema_sempre_existe_auto_explicacao_agmpabd}, if $\papfip$, then $H(\Theta, \varphi) \neq \emptyset$. In this case, given conditions (i)-(iii) of definition \ref{conj_hip_abdutivas_agmpabd}, $(\varphi \to \psi) \to \varphi \in H(\Theta, \varphi)$, for example, may be an abductive hypothesis for explaining $\varphi$.}. In other words, paraconsistency and the richness of the language of \textbf{RCbr} allow us to work with initial belief sets that, in the classical case, would be immediately rejected by the operation $\expfip$. It is important to note, however, that if the agent initially judges $\psi$ to be consistent, i.e., $\{\varphi \to \psi, \varphi \to \neg \psi, \circ \psi\} \subseteq \Theta$, then the abductive expansion operation fails\footnote{The initial belief set $\Theta$, however, is not rejected from the start for the operation $\expfip$ as in the classical case, even though $\{\varphi \to \psi, \varphi \to \neg \psi, \circ \psi\} \vdash_{\textbf{RCbr}} \neg \varphi$, according to \textit{reductio} rules valid in \textbf{mbC}. In this case, since $\Theta \in Th(\mathbb{L})$, we obtain $\{\varphi \to \psi, \varphi \to \neg \psi, \circ \psi, \neg \varphi\} \subseteq \Theta$. This is a case of \textit{weak abductive anomaly} (definition \ref{gatilhos_paraconsistentes_agmpabd}), which may be subject to the expansion operation $\expfip$ normally.}: by observation \ref{varphi_in_anel_agmpabd}, $\varphi \in \expfip$. By postulates $\postp1$ and $\postp2$, we obtain $\{\varphi \to \psi, \varphi \to \neg \psi, \circ \psi, \varphi \} \subseteq \expfip$ which, naturally, by postulate $\postp1$ again, results in $\expfip = \Theta_\bot$, which violates postulate $\postp4$. The same occurs in the case where $\circ \neg \psi \in \Theta$, since, recall, in \textbf{RCbr}, $\circ \psi \equiv \circ \neg \psi$ (property \ref{propriedades_relacao_consistencia_A_e_neg_A_e_preservacao_equiv_de_o_Cie_Cbr}, p.\pageref{propriedades_relacao_consistencia_A_e_neg_A_e_preservacao_equiv_de_o_Cie_Cbr}).

Finally, consider example (iv) $\papfip$ (definition \ref{definição_par_apropriado_expansao_paraconsistente_agmpabd}) and $\{\varphi \lor \psi, \neg \psi\} \subseteq \Theta$. Note that, in the classical case---that is, $\papfi$ and $\{\varphi \lor \psi, \til \psi\} \subseteq \Theta$---there is a situation similar to the previous example, namely, since $\Theta \in Th(\mathbb{L})$, then $\varphi \in \Theta$---because $\varphi \lor \psi, \til \psi \vdash_{\textbf{CPL}} \varphi$---which violates the conditions of $\papfi$. In other words, the pair $(\Theta, \varphi)$ is not appropriate for classical \textbf{AGM}$p_{abd}$ abductive expansion and is thus rejected beforehand. In the paraconsistent case, however, since $\Theta, \circ \psi \vdash_{\mathbb{L}} \varphi$---because $\varphi \lor \psi, \neg \psi, \circ \psi \vdash_{\textbf{RCbr}} \varphi$\footnote{It is also the case that $\varphi \lor \psi, \circ \psi \vdash_{\textbf{mbC}} \neg \psi \to \varphi$. By the \textit{Deduction Theorem}, we obtain $\varphi \lor \psi, \circ \psi, \neg \psi \vdash_{\textbf{mbC}} \varphi$. Evidently, therefore, this is also the case in \textbf{mbCciw} and \textbf{RCbr}.}---, $\Theta, \circ \psi \nvdash_{\mathbb{L}} \bot$ and $\circ \psi \nvdash_{\mathbb{L}} \varphi$, according to conditions (i)-(iii) of definition \ref{conj_hip_abdutivas_agmpabd}, we have that $\circ \psi \in H(\Theta, \varphi)$---that is, $\circ \psi$ becomes an abductive hypothesis for explaining $\varphi$---as well as $\circ \neg \psi$, since, again, in \textbf{RCbr}, $\circ \psi \equiv \circ \neg \psi$---a novelty evidently inconceivable in the classical case. Thus, by postulates $\postp1$ - $\postp4$ and observation \ref{varphi_in_anel_agmpabd}, we have that $\{\varphi \lor \psi, \neg \psi, \circ \psi, \varphi \} \subseteq \expfip$ becomes an admissible abductive expansion\footnote{This is a formal consequence of the paraconsistent \textbf{AGM}$p_{abd}$ abductive expansion operation, but its epistemic interpretation is naturally controversial. A sentence marked with the consistency operator $\circ$ by the agent should more adequately represent a more consolidated theory, not a hypothesis of a conjectural nature. At least with respect to the \textbf{AGM}$p_{abd}$ system, however, this feature is acceptable, insofar as it fulfills the function of establishing formal differences between the classical and paraconsistent cases.}.
\section{AGM\texorpdfstring{$p_{abd}$}{} abductive expansion  - construction}
\label{agmpabd_expansion_construction}

Analogously to the classical \textbf{AGM} system and similarly to Pagnucco's system (see \cite[p.108]{Pagnucco}), let us consider the following definitions.

\begin{definition} \label{definicao_supconj_maximal_agmpabd}
    Let $\papfig$ (definition \ref{definicao_par_apropriado_expansao_LPC_RCbr_agmpabd}) and $H(\Theta, \varphi)$ be the set of abductive hypotheses for $(\Theta, \varphi)$ (definition \ref{conj_hip_abdutivas_agmpabd}). A set $\Theta'$ is a maximal non-trivial \textbf{AGM}$p_{abd}$ superset of $\Theta$ with respect to $\varphi$ if, and only if:
\end{definition}

\begin{description}
\leftskip 40pt
    \item (i) $\Theta \subset \Theta'$ 
    \item (ii) $\Theta' \cap H(\Theta, \varphi) \neq \emptyset$ 
    \item (iii) $\bot \notin Cn_{\mathbb{L}}(\Theta')$ 
    \item (iv) There is no $\Theta'' \supset \Theta'$ that satisfies (i), (ii) and (iii). 
\end{description}


\begin{definition}[Surplus set]  \label{definicao_conj_excedente_agmpabd}
The surplus set of $\Theta$ with respect to $\varphi$, denoted $\excedfi$, is the set of all maximal non-trivial \textbf{AGM}$p_{abd}$ supersets of $\Theta$ with respect to $\varphi$ (definition \ref{definicao_supconj_maximal_agmpabd}) such that:

$\Theta' \in \excedfi$ if, and only if, $\Theta'$ satisfies conditions (i)-(iv) of definition \ref{definicao_supconj_maximal_agmpabd}.
  
\end{definition}

The surplus set $\excedfi$, therefore, contains the whole family of maximal non-trivial supersets satisfying (i)-(iv). As in the case of the postulates, there is the initial requirement that $\papfig$ and conditions (i)-(iii) are analogous to postulates $\postp2$ - $\postp4$. Requirement (iv) merely guarantees the maximality of the supersets $\Theta' \in \excedfi$. Just as the operation $\ooplusg$ defined by the postulates does not require a failure postulate (see Corollary \ref{corolario_saturado_neg_delta_in_Gamma}), since it always expands the initial epistemic state, the following results can also be obtained here:

\begin{remark}  \label{remark_superconjunto_maximal_conjunto_de_crencas_agmpabd}
    Any $\Theta' \in \excedfi$ (definition \ref{definicao_conj_excedente_agmpabd}) is a belief set.
\end{remark}

\begin{remark} \label{varphi_in_anel_conj_maximal_agmpabd}
   Let $\Theta' \in \excedfi$ be a maximal non-trivial \textbf{AGM}$p_{abd}$ superset of $\Theta$ with respect to $\varphi$ (definition \ref{definicao_supconj_maximal_agmpabd}). It is the case that $\varphi \in \Theta'\setminus\Theta$ for every $\Theta' \in \excedfi$. 
\end{remark}

\begin{lemma} \label{lema_alpha_ou_fi_in_supconj_maximal_agmpabd}
Let $\Theta' \in \excedfi$ be a maximal non-trivial \textbf{AGM}$p_{abd}$ superset of $\Theta$ with respect to $\varphi$ (definition \ref{definicao_supconj_maximal_agmpabd}). For $\alpha \in H(\Theta, \varphi)$, it is the case that $\alpha \lor \varphi \in \Theta'$ for every $\Theta' \in \excedfi$.
\end{lemma}

\begin{corollary} \label{corolario_alpha_ou_fi_sempre_satisfaz_ii_supconj_maximal_agmpabd}
    On the one hand, by lemma \ref{lema_alpha_ou_fi_in_supconj_maximal_agmpabd}, for $\alpha \in H(\Theta, \varphi)$, it is the case that $\alpha \lor \varphi \in \Theta'$ for every $\Theta' \subseteq \excedfi$. On the other hand, by lemma \ref{lema_alpha_ou_fi_in_hip_abdutivas_agmpabd}, if it is the case that $\papfig$---a requirement of definition \ref{definicao_supconj_maximal_agmpabd}---then it is also the case that $\alpha \lor \varphi \in H(\Theta, \varphi)$. Thus, we have that $\alpha \lor \varphi \in \Theta' \cap H(\Theta, \varphi)$ for any $\Theta' \in \excedfi$ and $\alpha \in H(\Theta, \varphi)$. 
\end{corollary}

\begin{lemma}  \label{lema_conjunto_dos_superconjuntos_maximais_nunca_vazio_agmpabd}
   Let $\excedfi$ be the surplus set of $\Theta$ with respect to $\varphi$ (definition \ref{definicao_conj_excedente_agmpabd}). It is the case that $\excedfi \neq \emptyset$.
\end{lemma}

\begin{corollary} \label{corolario_alpha_ou_fi_in_interseccao_conj_excedente_agmpabd}
By lemma \ref{lema_alpha_ou_fi_in_supconj_maximal_agmpabd}, for $\alpha \in H(\Theta, \varphi)$, it is the case that $\alpha \lor \varphi \in \Theta'$ for every $\Theta' \in \excedfi$. Therefore, $\alpha \lor \varphi \in \bigcap \excedfi$.
\end{corollary}

Let us now consider the following well-known definitions from the \textbf{AGM} literature---evidently, definitions that disregard, thanks to Lemma \ref{lema_conjunto_dos_superconjuntos_maximais_nunca_vazio_agmpabd}, the case $\excedfi = \emptyset$:

\begin{definition}[Partial meet selection function $\gamma$]  \label{definicao_funcao_selecao_partial_meet_agmpabd}
Let $\excedfi$ be the surplus set of $\Theta$ with respect to $\varphi$ (definition \ref{definicao_conj_excedente_agmpabd}). A partial meet selection function is a function $\gamma : Th(\mathbb{L}) \times \mathcal{L}_{\Sigma_*} \to \wp(Th(\mathbb{L})) \setminus \{\emptyset\}$ such that, for every $\Theta$ and $\varphi$ such that $\papfig$, it selects some elements of $\excedfi$, i.e.:  
    $$\gamma (\Theta,\varphi) \subseteq \excedfi$$
\end{definition}

\begin{definition}[\textbf{AGM}$p_{abd}$ \textit{partial meet} abductive expansion operation] \label{definicao_expansao_abdutiva_partial_meet_agmpabd}
We define $\ooplusg$ as a \textbf{AGM}$p_{abd}$ \textit{partial meet} abductive expansion operation of $\Theta$ with respect to $\varphi$ by:
$$\expfip = \bigcap \gamma (\Theta, \varphi)$$
\end{definition}

As expected, the following more relevant result can also be obtained in this system:

\begin{theorem}  \label{teorema_partial_meet_teta1-teta5_agmpabd}
    For every $\Theta$ and $\varphi$ such that $\papfig$, $\ooplusg$ is a \textbf{AGM}$p_{abd}$ partial meet abductive expansion operation of $\Theta$ with respect to $\varphi$---definition \ref{definicao_expansao_abdutiva_partial_meet_agmpabd}---if, and only if, $\ooplusg$ satisfies postulates $\postp1$ - $\postp5$ from definition \ref{definicao_postulados_expansao_abdutiva_agmpabd} of abductive expansion of $\Theta$ with respect to $\varphi$.
\end{theorem}

In order to capture the notion that the function $\gamma$ selects those maximal non-trivial supersets of $\Theta$ that do not imply $\varphi$, which are \textit{worth retaining} at least as much as any other---since they are considered ``more important'' or ``better''---in the \textit{partial meet} abductive expansion operation of the \textbf{AGM}$p_{abd}$ system, an ordering is established among the elements $\Theta' \in \Theta \top \varphi$. Just as in the classical \textbf{AGM} system and in Pagnucco's system\footnote{According to \cite[p.112]{Pagnucco}.}, a transitively relational \textbf{AGM}$p_{abd}$ \textit{partial meet} abductive expansion operation---which also encompasses postulates ($\Theta^\ooplus6$) and ($\Theta^\ooplus7$)---can also be established here. Let us therefore consider the following definition.

\vbox{\begin{definition}[Marking-off identity]
\mbox{}\

\centering
$\gamma_\leqslant(\Theta, \varphi) = \{\Theta' \in \excedfi : \Theta'' \leqslant \Theta'$ for every $\Theta'' \in \excedfi\}$. 
\label{definicao_marking_off_agmpabd}
\end{definition}}

The relationality of $\leqslant$ is already given in definition \ref{definicao_marking_off_agmpabd} itself, and the transitivity of $\leqslant$ is notably conferred upon it by assuming the property that if $\Theta' \leqslant \Theta''$ and $\Theta'' \leqslant \Theta'''$, then $\Theta' \leqslant \Theta'''$. No further property needs to be imposed on the relation $\leqslant$ at this point for the following lemmas to be demonstrated\footnote{I emphasize that, since the supplementary postulates $\postp6$ and $\postp7$ are strictly the same as in the \textbf{AGM}$p_{abd}$ system, the lemmas \ref{lema_excedente_ou_igual_excedente_uniao_B5_agmpabd}, \ref{lema_gamma_excedente_subseteq_gamma_excedente_varphi1_varphi2_B6_agmpabd} and \ref{lema_gamma_excedente_subseteq_bigcap_gamma_excedente_B7_agmpabd}, reproduced in Appendix \ref{apendiceA}, are analogous, with some differences and adjustments, to those proved by Pagnucco \cite[lemmas B.5, B.6 and B.7, p.226-227]{Pagnucco}.}.

\begin{lemma} \label{lema_partial_meet_relacional_teta6_agmpabd}
Any relational partial meet abductive expansion function, i.e., $\expfip = \bigcap \gamma_\leqslant(\Theta,\varphi)$ (definitions \ref{definicao_expansao_abdutiva_partial_meet_agmpabd} and \ref{definicao_marking_off_agmpabd}), satisfies postulate $\postp6$ from definition \ref{definicao_postulados_expansao_abdutiva_agmpabd}.
\end{lemma}

\begin{lemma} \label{lema_partial_meet_transitivamente_relacional_teta7_agmpabd}
Any transitively relational partial meet abductive expansion function (definitions \ref{definicao_expansao_abdutiva_partial_meet_agmpabd} and \ref{definicao_marking_off_agmpabd}) satisfies postulate $\postp7$ from definition \ref{definicao_postulados_expansao_abdutiva_agmpabd}.
\end{lemma}

Finally, as the most important result of this paper, the following theorem can be demonstrated, both in the classical and in the paraconsistent cases.

\begin{theorem}  \label{teorema_Tabd_partial_meet_transitivamente_relacional_teta1-teta7_agmpabd}
    For every $\Theta$ and $\varphi$ such that $\papfig$, $\ooplusg$ is a transitively relational \textbf{AGM}$p_{abd}$ \textit{partial meet} abductive expansion operation (definitions \ref{definicao_expansao_abdutiva_partial_meet_agmpabd} and \ref{definicao_marking_off_agmpabd}) if, and only if, $\ooplusg$ satisfies postulates $\postp1$ - $\postp7$ from definition \ref{definicao_postulados_expansao_abdutiva_agmpabd} of abductive expansion on $\Theta$.
\end{theorem}

\section{Conclusion and Future Work}
\label{conclusion}

In this paper, the paraconsistent \textbf{AGM}-style abductive expansion operation was developed---to the best of my knowledge, the first of its kind in the \textbf{AGM} literature. Specifically, its postulates and its transitively relational \textit{partial meet} construction. This operation was elaborated from rationality criteria specifically designed to satisfy some important philosophical demands of both abductive reasoning, especially those found in the thought of C.S. Peirce, its creator, and explanation; both commonly neglected by the formal logic literature. A general taxonomy, inspired by Aliseda \cite{Aliseda}, which attempts to capture the processual aspect of abductive reasoning---from the initial surprise in the face of a surprising phenomenon, which demands adequately explanatory hypotheses, to the expanded epistemic states---was designed accordingly. Thus, the abductive expansion operation developed in this paper, even in its classical aspect, although formally based on the abductive expansion operation originally created by Pagnucco \cite{Pagnucco}, implements a different philosophical perspective: for the author, abduction, as a mere logical form\footnote{Definition found in \cite[79]{Pagnucco}.}, adds interesting features to the \textbf{AGM} system\footnote{This view is in accordance with Pagnucco's proposal and can be easily noted in the very term ``within'' in the title of his thesis: ``The role of abductive reasoning within the process of belief revision''.}. On the other hand, for me, it is the \textbf{AGM} system that adds some interesting formal aspects to abduction---or rather, abductive reasoning is a philosophically complex inferential process, and some of these aspects can be formally captured by the interaction between the taxonomy and the \textbf{AGM} operation. Moreover, and as the most relevant formal development, thanks to the logic \textbf{RCbr}, the abductive expansion operation developed in this paper has a paraconsistent aspect, capable of providing unique capabilities to the operation, as seen in Subsection \ref{distinctions_classic_paraconsistent}.

The content presented in this paper is part of an ongoing PhD thesis. Thus, the paraconsistent abductive expansion operation is part of a broader system called \textbf{AGM}$p_{abd}$. Other operations, such as internal and external revision---the latter being possible only due to paraconsistency---can be obtained, with many limitations\footnote{The same limitations were obtained by Pagnucco, in the sense that only a few simple postulates can be obtained, since the non-monotonicity of the abductive expansion operation entails that an adequate inclusion relation between the expansion and revision operations, as occurs in classical \textbf{AGM} revision, cannot be fully obtained. See \cite[p.165-166]{Pagnucco}.}. Moreover, a construction based on an \textit{abductive entrenchment} order, also with its paraconsistent aspect, thanks to the underlying logic \textbf{RCbr}, can be implemented. A new system called \textbf{AGM}$\circ_{abd}$---which will be the subject of a second paper---is also being developed to deal, in a more granular and particular way, with the triggers of abductive strong novelty, weak novelty and weak anomaly (definition \ref{gatilhos_paraconsistentes_agmpabd} and figure \ref{figura_processo_abdutivo_paraconsistente_agmpabd}), assigning differentiated roles to the paraconsistent operators $\neg$ and $\circ$. Two other future works, associated with these systems, will also be developed. They concern the dynamics of preferences among hypotheses, contradictory or not---specifically in the association between the logic \textbf{RCbr} and the modal logic \textbf{DBL} (\textit{Dynamic Betterness Logic} of Fenrong Liu \cite{FEIRONG_BOOK_2011})---and the non-monotonicity of the paraconsistent abductive expansion operation---similarly to the results obtained by Makinson and Gärdenfors, \cite{MAK_NOMON_PATTERN} and \cite{GARD_MAK_NONMON_EXPEC}, for classical \textbf{AGM} revision.

\begin{appendices}

\section{Supplementary Materials}
\label{apendiceA}

\subsection{Tarskian Logic Properties}

Some general properties about Tarskian logic that will be useful in many proofs. As usually, we should also assume that $\varphi \leftrightarrow \psi$ is the same as \mbox{$(\varphi \to \psi) \land (\psi \to \varphi)$} and $\varphi \equiv \psi$ is the same as $\vdash \varphi \leftrightarrow \psi$. 

\begin{definition}[Tarskian Logic] \label{definicao_logica_tarskiana}
    A logic $\mathbb{L} = \langle \mathcal{L}_\Sigma, Cn \rangle$ defined over a language $\mathcal{L}_\Sigma$ and a consequence relation $Cn$ - or $\vdash$ - is Tarskian if it satisfies the following three properties\protect\footnotemark, for all $\Gamma \cup \Delta \cup \{\varphi\} \subseteq \mathcal{L}_\Sigma$:

        \begin{description}
            \item (i) $\Gamma \subseteq Cn(\Gamma)$  (\textit{inclusion})             
            \item (ii) $Cn(Cn(\Gamma)) \subseteq Cn(\Gamma)$ (\textit{iteration})             
            \item (iii) If $\Gamma \subseteq \Delta$, then $Cn(\Gamma) \subseteq Cn(\Delta)$  (\textit{monotonicity)}
            
        \end{description}
        or
        \begin{description}
            \item (i) If $\varphi \in \Gamma$, then $\Gamma \vdash \varphi$ (\textit{reflexivity})
            \item (ii) If $\Delta \vdash \varphi$ and $\Gamma \vdash \psi$ for any $\psi \in \Delta$, then $\Gamma \vdash \varphi$ (\textit{cut})
            \item (iii) If $\Gamma \vdash \varphi$ and $\Gamma \subseteq \Delta$, then $\Delta \vdash \varphi$ (\textit{monotonicity)}            
        \end{description}
 \end{definition}

\footnotetext{I believe it is important to maintain both notations, since they can be frequently found in \textbf{AGM} literature and they will be used interchangeably in this paper, considering that the Tarskian consequence operator $Cn: \wp(\mathcal{L}_\Sigma) \to \wp(\mathcal{L}_\Sigma)$ is defined as follows: $Cn(\Gamma) = \{\varphi \in \mathcal{L}_\Sigma : \Gamma \vdash \varphi\}$. In particular, therefore, $\vdash \varphi$ is the same as $\varphi \in Cn(\emptyset)$.}

\begin{definition}[Standard logic]
A logic $\mathbb{L}$ is considered standard if it is Tarskian, finitary, and structural, that is, if it satisfies, in addition to the three properties described in Definition \ref{definicao_logica_tarskiana}, two other properties, respectively:
\begin{description}

    \item (iv) If $\Gamma \vdash \varphi$, then there exists a finite subset $\Gamma_0 \subseteq \Gamma$ such that $\Gamma_0 \vdash \varphi$ (\textit{compactness});
    \item (v) If $\Gamma \vdash \varphi$, then $\sigma[\Gamma] \vdash \sigma[\varphi]$ for every substitution $\sigma$ of formulas for variables (\textit{substitution}).
\end{description}
\label{definicao_logica_tarskiana_standard}
\end{definition}

\begin{definition}[Supraclassical Tarskian logic]
A Tarskian logic \mbox{$\mathbb{L} = \langle \mathcal{L}_\Sigma, Cn \rangle$} is considered supraclassical if, given a classical consequence operator $C_{CPL}$, it is the case that $C_{CPL}(\Gamma) \subseteq Cn(\Gamma)$ for every set of formulas $\Gamma$. That is, for every formula $\varphi$ that can be classically derived from $\Gamma$, i.e., $\Gamma \vdash_{CPL} \varphi$, then $\varphi \in Cn(\Gamma)$.
\label{definicao_logica_tarskiana_supraclassuca}
\end{definition}

\begin{properties}[Valid properties in a supraclassical Tarskian logic $\mathbb{L}$\footnote{Including \textbf{CPL} and \textbf{mbC} (and its extensions). Of course, in a classical notation, the operator $\neg$ found in (iii) must be read as the strong classical negation, but, in the case os \textbf{mbC} and its extensions, $\neg$ is the paraconsistent negation. Property (iii) can be obtained, since $\psi \lor \neg \psi$ is an axiom in the \textbf{mbc} axiomatic scheme.}] \mbox{}
\begin{description}
    \item (i) $\Gamma \cup \{\psi\} \vdash \varphi$ if and only if $\Gamma \vdash \psi \to \varphi$ (\textit{Deduction Theorem});
    \item (ii) If $\Gamma \cup \{\psi\} \vdash \varphi$ and $\Gamma \cup \{\delta\} \vdash \varphi$, then $\Gamma \cup \{\psi \lor \delta\} \vdash \varphi$  (\textit{Disjunction of Premises});
    \item (iii) If $\Gamma \cup \{\psi\} \vdash \varphi$ and $\Gamma \cup \{\neg \psi\} \vdash \varphi$, then $\Gamma \vdash \varphi$ (\textit{Proof by Cases}).
\end{description}
\label{propriedades_teorema_da_deducao_disjuncao_das_premissas}
\end{properties}

The following definitions and properties are well-known in literature.

\begin{definition}[$\delta$-saturated set] \label{definicao_delta_saturado}
For some Tarskian logic $\mathbb{L}$ over the language $\mathcal{L}_\Sigma$, let $\Gamma \cup \{\delta\} \subseteq \mathcal{L}_\Sigma$. We say that $\bar\Gamma$ is a maximal consistent set with respect to $\delta$ in $\mathbb{L}$ - or $\delta$-saturated - if it satisfies the following properties:
\begin{description}
\item (i) $\bar\Gamma \nvdash \delta$; 
\item (ii) if $\psi \not\in \bar\Gamma$ then $\bar\Gamma,\psi \vdash \delta$.
\end{description}
\end{definition}

\begin{properties}
\label{proposicao_propriedade_delta_saturado}
Let $\bar\Gamma$ be a $\delta$-saturated set (Definition \ref{definicao_delta_saturado}). Then, $\bar\Gamma$ satisfies the following properties\footnote{Again, including \textbf{CPL} and \textbf{mbC} (and its extensions). When considering \textbf{CPL}, the operator $\neg$ must be interpreted as classical strong negation. In the case of \textbf{mbC} and its extensions, on the other hand, $\neg$ must be interpreted as paraconsistent negation. It is important to emphasize that, in the case of \textbf{CPL}, property (ii) can be verified in both directions, i.e., $\psi \notin \bar\Gamma$ if and only if $\neg \psi \in \bar\Gamma$. To verify the \textit{only if} direction, it suffices to suppose $\neg \psi \in \bar\Gamma$ and, by \textit{reductio}, suppose $\psi \in \bar\Gamma$. Since, by \textit{ex falso}, $\psi, \neg \psi \vdash \delta$, then we have $\bar\Gamma \vdash \delta$, a contradiction with the very definition of $\delta$-saturated. Hence, $\psi \notin \bar\Gamma$. I choose, however, to present property (ii) in this way, with only one direction of the conditional, because, thus described, it (and the corollary \ref{corolario_saturado_neg_delta_in_Gamma} below, by the way) is valid in the paraconsistent logic \textbf{mbC}---and, therefore, in its extensions---even with the paraconsistent negation $\neg$. The proof is identical, since the rule of \textit{proof by cases}---property \ref{propriedades_teorema_da_deducao_disjuncao_das_premissas} (iii)---holds in \textbf{mbC}.\label{nota_rodape_saturado_vale_mbC}}:
\begin{description}

    \item (i) $\bar\Gamma$ is a set closed under logical consequences; 
    \item (ii) if $\psi \notin \bar\Gamma$, then $\neg \psi \in \bar\Gamma$.
\end{description}
\end{properties}

\begin{proof}\mbox{}\\
    (i) We need to show that $\bar\Gamma \vdash \psi$ if and only if $\psi \in \bar\Gamma$.\\ 
    (If:) If $\psi \in \bar\Gamma$, since $\mathbb{L}$ is a Tarskian logic, by \textit{inclusion}, we directly have that $\bar\Gamma \vdash \psi$.\\
    (Only if:) We need to show that if $\bar\Gamma \vdash \psi$, then $\psi \in \bar\Gamma$. By \textit{contraposition}, this is equivalent to showing that if $\psi \notin \bar\Gamma$, then $\bar\Gamma \nvdash \psi$. Suppose, therefore, that $\psi \notin \bar\Gamma$. By Definition \ref{definicao_delta_saturado} (ii), $\bar\Gamma,\psi \vdash \delta$. By \textit{reductio}, let us assume that $\bar\Gamma \vdash \psi$. In this case, $\bar\Gamma \vdash \delta$, a contradiction with Definition \ref{definicao_delta_saturado} (i). Therefore, $\bar\Gamma \nvdash \psi$.\\[1mm]
    (ii) Suppose that $\psi \notin \bar\Gamma$. By Definition \ref{definicao_delta_saturado} (ii), $\bar\Gamma,\psi \vdash \delta$. In this case, suppose, by \textit{reductio}, that $\neg \psi \notin \bar\Gamma$. Again, by Definition \ref{definicao_delta_saturado} (ii), $\bar\Gamma,\neg \psi \vdash \delta$. We therefore have, by \textit{proof by cases}, that $\bar\Gamma \vdash \delta$, a contradiction with Definition \ref{definicao_delta_saturado} (i). Hence, $\neg \psi \in \bar\Gamma$.
\end{proof}

\begin{corollary} \label{corolario_saturado_neg_delta_in_Gamma}
    Let $\bar\Gamma$ be a $\delta$-saturated set. Since $\delta \notin \bar\Gamma$, by Proposition \ref{proposicao_propriedade_delta_saturado}, then $\neg \delta \in \bar\Gamma$.
\end{corollary}


\begin{theorem}[Lindenbaum-Łoś Theorem\protect\footnotemark] \label{teorema_lindenbaum_los}
Let $\mathbb{L}$ be a finitary Tarskian logic over the language $\mathcal{L}_\Sigma$. Let $\Gamma \cup \{\delta\} \subseteq \mathcal{L}_\Sigma$ such that $\Gamma \nvdash \delta$. There exists a set $\Delta$ such that $\Gamma \subseteq \Delta \subset \mathcal{L}_\Sigma$ with $\Delta$ maximally consistent (non-trivial) with respect to $\delta \in \mathcal{L}_\Sigma$ - that is, $\delta$-saturated (Definition \ref{definicao_delta_saturado}).    
\end{theorem}

\footnotetext{The proof of this theorem can be found in \cite[p.54, Theorem 22.2]{Wojcicki1984} and, in an adapted version, \cite[p.37, Theorem 2.2.6]{CONCAR_2016}.}

\section{Sections \ref{agmpabd_preliminaries}, \ref{agmpabd_expansion_postulates} and \ref{agmpabd_expansion_construction} proofs}
\label{apendiceB}

\noindent\textbf{Observation \ref{remark_par_apropriado_expansao_negacao_forte_agmpabd}}:
\textit{It is the case that $\papfig$ (definition \ref{definicao_par_apropriado_expansao_LPC_RCbr_agmpabd}) if and only if $\varphi \notin \Theta$ and $\til\varphi \notin \Theta$}.

\begin{proof}
    (If:) By definition \ref{definicao_par_apropriado_expansao_LPC_RCbr_agmpabd}, $\papfig$ is the case, equivalently, if: (i) $\varphi \notin \Theta$ and (ii) ($\circ \varphi \notin \Theta$ and $\neg \varphi \notin \Theta$) or ($\circ \varphi \in \Theta$ and $\neg \varphi \notin \Theta$) or ($\circ \varphi \notin \Theta$ and $\neg \varphi \in \Theta$). From this, it follows that $\papfig$ is the case if (i) $\varphi \notin \Theta$ and (ii) it is not the case that $\circ \varphi \in \Theta$ and $\neg \varphi \in \Theta$. Since $\Theta \in Th(\mathbb{L})$, then $(\circ \varphi \land \neg \varphi) \notin \Theta$. Since $\til\varphi \equiv_{\mathbb{L}} \circ \varphi \land \neg \varphi$, then $\til\varphi \notin \Theta$.\\[1mm]
    (Only if:) Consider $\varphi \notin \Theta$ and $\til\varphi \notin \Theta$. Since $\til\varphi \equiv_{\mathbb{L}} \circ \varphi \land \neg \varphi$, then $(\circ \varphi \land \neg \varphi) \notin \Theta$. Since $\Theta \in Th(\mathbb{L})$, then we have three possibilities: $\circ \varphi \notin \Theta$ or $\neg \varphi \notin \Theta$ or ($\circ \varphi \notin \Theta$ and $\neg \varphi \notin \Theta$), which exactly reflects the conditions of definition \ref{definicao_par_apropriado_expansao_LPC_RCbr_agmpabd}.    
\end{proof}

\noindent\textbf{Lemma \ref{lema_sempre_existe_auto_explicacao_agmpabd}} (Partial self-explanation):  
    \textit{If $\papfig$ is the case (definition \ref{definicao_par_apropriado_expansao_LPC_RCbr_agmpabd}), then there always exists some $\alpha \in \mathcal{L}_{\Sigma_*}$ that satisfies criteria (i)-(iii) of definition \ref{conj_hip_abdutivas_agmpabd}. Hence, $H(\Theta, \varphi) \neq \emptyset$.}
\begin{proof}\mbox{}\\
    \noindent
        Suppose $\papfig$ is the case---that is, $\varphi \notin \Theta$ and $\til \varphi \notin \Theta$ (definition \ref{definicao_par_apropriado_expansao_LPC_RCbr_agmpabd}). Since $\Theta \in Th(\mathbb{L})$, then $\Theta \neq \emptyset$. Consider, therefore, arbitrary $\delta, \alpha \in \mathcal{L}_{\Sigma_*}$ such that $\delta \in \Theta$ and $\alpha = \delta \to \varphi$. We need to show that $\alpha$ satisfies criteria (i)-(iii) of definition \ref{lema_sempre_existe_auto_explicacao_agmpabd}. Since, by \textit{modus ponens}, $\Theta \cup \{\delta \to \varphi\} \vdash_{\mathbb{L}} \varphi$, then $\alpha$ satisfies criterion (i). Since $\papfig$ is the case, then $\til \varphi \notin \Theta$. Thus, $\Theta \cup \{\delta \to \varphi\} \nvdash_{\mathbb{L}} \bot$. So $\alpha$ satisfies criterion (ii). Finally, $\delta \to \varphi \nvdash_{\mathbb{L}} \varphi$ and condition (iii) is satisfied. Therefore, if $\papfig$ is the case, then there always exists an $\alpha \in \mathcal{L}_{\Sigma_\circ}$ such that $\alpha \in H(\Theta, \varphi)$. Hence, $H(\Theta, \varphi) \neq \emptyset$.
\end{proof}

\

\noindent\textbf{Observation \ref{alpha_notin_theta_agmpabd}}: \textit{Let $H(\Theta, \varphi)$ be the set of abductive hypotheses for the pair $(\Theta,\varphi)$ (definition \ref{conj_hip_abdutivas_agmpabd}). It is the case that \mbox{$\Theta \cap H(\Theta,\varphi)=\emptyset$}. }
\begin{proof}\mbox{}\\
        From definition \ref{conj_hip_abdutivas_agmpabd}, we have that $\varphi \notin \Theta$. Suppose, by \textit{reductio}, that $\Theta \cap H(\Theta,\varphi) \neq \emptyset$. In this case, from condition (i) $\Theta, \alpha \vdash_{\mathbb{L}} \varphi$, we infer that $\Theta \vdash_{\mathbb{L}} \varphi$. Since $\Theta \in Th(\mathbb{L})$, then $\varphi \in \Theta$, a contradiction. Hence, as required, $\Theta \cap H(\Theta,\varphi)=\emptyset$.
\end{proof}

\

\noindent\textbf{Lemma \ref{conj_hipoteses_abdutivas_iguais_varphi1_varphi2_agmpabd}}: \textit{ Let $\papfig$ and $\pappsig$. Let $H(\Theta, \varphi)$ and $H(\Theta, \psi)$ be the sets of abductive hypotheses for $\varphi$ and $\psi$, respectively. It is the case that if $\vdash_{\mathbb{L}} \varphi \leftrightarrow \psi$, then $H(\Theta, \varphi) = H(\Theta, \psi)$.}  

\begin{proof}\mbox{}\\
    \noindent
    Since $\papfig$ and $\pappsig$ are the case, then, by lemma \ref{lema_sempre_existe_auto_explicacao_agmpabd}, $H(\Theta, \varphi) \neq \emptyset$ and $H(\Theta, \psi) \neq \emptyset$. This condition, therefore, does not need to be considered.\\[1mm]
    Consider, therefore, $\vdash_{\mathbb{L}} \varphi \leftrightarrow \psi$. We need to show that (1) $H(\Theta, \varphi) \subseteq H(\Theta, \psi)$ and (2) $H(\Theta, \psi) \subseteq H(\Theta, \varphi)$:

    {\leftskip=1.5cm \noindent
    (1) $H(\Theta, \varphi) \subseteq H(\Theta, \psi)$. We need to show that if $\alpha \in H(\Theta, \varphi)$, then $\alpha \in H(\Theta, \psi)$. If $\alpha \in H(\Theta, \varphi)$, then, by the three conditions of definition \ref{conj_hip_abdutivas_agmpabd}, $(i) \ \Theta, \alpha \vdash_{\mathbb{L}} \varphi, \ (ii) \ \Theta, \alpha \nvdash_{\mathbb{L}} \bot \ \mbox{ and } \ (iii) \ \alpha \nvdash_{\mathbb{L}} \varphi$. Since $\vdash_{\mathbb{L}} \varphi \leftrightarrow \psi$, then $\varphi$ and $\psi$ can be substituted in any context---including in the paraconsistent case, since \textbf{RCbr} satisfies the \textit{replacement} property. In particular, therefore, it is also the case that $(i) \ \Theta, \alpha \vdash_{\mathbb{L}} \psi, \ (ii) \ \Theta, \alpha \nvdash_{\mathbb{L}} \bot \ \mbox{ and } \ (iii) \ \alpha \nvdash_{\mathbb{L}} \psi$ for any $\varphi, \psi \in \mathcal{L}_{\Sigma_*}$. Hence, $\alpha \in H(\Theta, \psi)$.

    \noindent
    (2) $H(\Theta, \psi) \subseteq H(\Theta, \varphi)$. Exactly the same reasoning as in (1).
     \par }
        
    Therefore, $H(\Theta, \varphi) = H(\Theta, \psi)$.   
\end{proof}

\begin{remark} \label{remark_theta_varphi1_varphi2_nao_e_o_caso_agmpabd}
Let $\papfig$ and $\pappsig$. Let $H(\Theta, \varphi)$ and $H(\Theta, \psi)$ be the sets of abductive hypotheses for $\varphi$ and $\psi$, respectively. If $\Theta \vdash_{\mathbb{L}} \varphi \leftrightarrow \psi$, but $\nvdash_{\mathbb{L}} \varphi \leftrightarrow \psi$, then $H(\Theta, \varphi) \neq H(\Theta, \psi)$.
\end{remark}
\begin{proof}\mbox{}\\
Consider $\Theta \vdash_{\mathbb{L}} \varphi \leftrightarrow \psi$ and $\nvdash_{\mathbb{L}} \varphi \leftrightarrow \psi$. Since $\Theta \in Th(\mathbb{L})$, then $\varphi \leftrightarrow \psi \in \Theta$. Thus, by definition \ref{conj_hip_abdutivas_agmpabd}, $(i) \ \Theta, \psi \vdash_{\mathbb{L}} \varphi, \ (ii) \ \Theta, \psi \nvdash_{\mathbb{L}} \bot \ \mbox{ and } \ (iii) \ \psi \nvdash_{\mathbb{L}} \varphi$. Hence, $\psi \in H(\Theta, \varphi)$. But $\varphi \notin H(\Theta, \varphi)$ since $\varphi \vdash_{\mathbb{L}} \varphi$, which violates criterion (iii). The same reasoning can be employed to conclude that $\varphi \in H(\Theta, \psi)$, but $\psi \notin H(\Theta, \psi)$. Therefore, $H(\Theta, \varphi) \neq H(\Theta, \psi)$.      
\end{proof}

\noindent\textbf{Observation \ref{varphi_in_anel_agmpabd}}: \textit{If $\expfip$ satisfies postulates $\postp1$ - $\postp3$ from definition \ref{definicao_postulados_expansao_abdutiva_agmpabd}, then $\varphi \in \expfip\setminus\Theta$.}

 \begin{proof}\mbox{}\\
    Since, by definition \ref{definicao_postulados_expansao_abdutiva_agmpabd}, $\papfig$ is the case, by lemma \ref{lema_sempre_existe_auto_explicacao_agmpabd}, $H(\Theta, \varphi) \neq \emptyset$. By postulate $\postp3$, therefore, there must be at least one abductive hypothesis $\alpha \in \expfip$. Since, by postulate $\postp2$, $\Theta \subset \expfip$ and by postulate $\postp1$ $\expfip \in Th(\mathbb{L})$, then $\varphi \in \expfip$ (since, by condition (i) of definition \ref{conj_hip_abdutivas_agmpabd} of $H(\Theta, \varphi)$, it is the case that $\Theta, \alpha \vdash_{\mathbb{L}} \varphi$). Since, again, $\papfig$ is the case, then $\varphi \notin \Theta$. Hence, $\varphi \in \expfip\setminus\Theta$.
    \end{proof}

\

\noindent\textbf{Lemma \ref{lema_alpha_ou_fi_in_hip_abdutivas_agmpabd}} \textit{Let $\papfig$ and $H(\Theta, \varphi)$ be the set of abductive hypotheses for $(\Theta, \varphi)$. If $\alpha \in H(\Theta,\varphi)$, then it is the case that $\alpha \lor \varphi \in H(\Theta, \varphi)$.}   
\begin{proof} \mbox{}\\
    By lemma \ref{lema_sempre_existe_auto_explicacao_agmpabd}, $H(\Theta, \varphi) \neq \emptyset$. Therefore, there is an $\alpha \in \mathcal{L}_{\Sigma_*}$ such that $\alpha \in H(\Theta,\varphi)$. In these terms, we need to show that $\alpha \lor \varphi$ satisfies conditions (i), (ii) and (iii) of definition \ref{conj_hip_abdutivas_agmpabd}. \\[1mm] 
    Given that, by condition (i) of definition \ref{conj_hip_abdutivas_agmpabd}, $\Theta, \alpha \vdash_{\mathbb{L}} \varphi$, and it is the case that $\Theta, \varphi \vdash_{\mathbb{L}} \varphi$, by \textit{disjunction of premises}---a rule valid in \textbf{mbC} and, therefore, in \textbf{RCbr}---we have that $\Theta \cup \{\alpha \lor \varphi\} \vdash_{\mathbb{L}} \varphi$ (i). Suppose now, by \textit{reductio}, that $\Theta \cup \{\alpha \lor \varphi\} \vdash_{\mathbb{L}} \bot$. Since $\{\alpha\} \vdash_{\mathbb{L}} \alpha \lor \varphi$, then (considering, evidently, that \textbf{RCbr}, like any \textbf{LFI}, is a standard Tarskian logic, according to definition \ref{definicao_logica_tarskiana_standard}, p. \pageref{definicao_logica_tarskiana_standard}), by \textit{cut}, we infer that $\Theta, \alpha \vdash_{\mathbb{L}} \bot$, a contradiction with condition (ii). Hence, $\Theta \cup \{\alpha \lor \varphi\} \nvdash_{\mathbb{L}} \bot$ (ii). Finally, it is the case that $\{\alpha \lor \varphi\} \nvdash_{\mathbb{L}} \varphi$ (iii). Therefore, $\alpha \lor \varphi \in H(\Theta, \varphi)$, as required.
\end{proof}

\

\noindent\textbf{Lemma \ref{lema_alpha_ou_fi_in_expansao_agmpabd}} \textit{If $\expfip$ satisfies postulates $\postp1$ - $\postp3$ from definition \ref{definicao_postulados_expansao_abdutiva_agmpabd}, then, for $\alpha \in H(\Theta, \varphi)$, it is the case that $\alpha \lor \varphi \in \expfip$.}

\begin{proof} \mbox{}\\
    Since, by definition \ref{definicao_postulados_expansao_abdutiva_agmpabd}, $\papfig$ is the case, then, by lemma \ref{lema_sempre_existe_auto_explicacao_agmpabd}, $H(\Theta, \varphi) \neq \emptyset$. Therefore, there is an $\alpha \in \mathcal{L}_{\Sigma_*}$ such that $\alpha \in H(\Theta,\varphi)$. By lemma \ref{varphi_in_anel_agmpabd}, we have that $\varphi \in \expfip$. Since, by postulate $\postp1$, $\expfip \in Th(\mathbb{L})$ and it is the case that $\varphi \vdash_{\mathbb{L}} \alpha \lor \varphi$, then, for $\alpha \in H(\Theta, \varphi)$, it is the case that $\alpha \lor \varphi \in \expfip$.    
\end{proof}

\

\begin{properties} \label{proposition_propriedades_teta6_agmpabd}
The following properties are consequences of postulate $\postp6$ in the presence of $\postp1$ - $\postp5$:
\begin{description}

\leftskip 40pt
    \item (i) $\expfip \cap \exppsip \subseteq \expfilorp$ 
    \item (ii) If $\varphi \in \expfilorp$, then $\expfip \subseteq \expfilorp$
\end{description}
\end{properties}
\begin{proof} \mbox{}\\
(i) $\expfip \cap \exppsip \subseteq \expfilorp$:

{\leftskip=1.5cm \noindent
 We need to show that if $\delta \in \expfip \cap \exppsip$, then $\delta \in \expfilorp$. If $\delta \in \expfip \cap \exppsip$, then $\delta \in \expfip$ and $\delta \in \exppsip$. Consider the first case. By postulate $\postp6$, $\delta \in Cn_{\mathbb{L}}(\expfilorp \cup \{\varphi\})$ and, by the \textit{deduction theorem}, we obtain $\varphi \to \delta \in Cn_{\mathbb{L}}(\expfilorp)$. By postulate $\postp1$, we have (a) $\varphi \to \delta \in \expfilorp$. Analogously, from the second case where $\delta \in \exppsip$, we have, by $\postp6$, $\delta \in Cn_{\mathbb{L}}(\expfilorp \cup \{\psi\})$ and, with the same reasoning, we obtain (b) $\psi \to \delta \in \expfilorp$. From (a) and (b) and, again, $\postp1$, we obtain $(\varphi \to \delta) \land (\psi \to \delta) \in Cn_{\mathbb{L}}(\expfilorp)$. By \textit{disjunction of premises}, we have $(\varphi \lor \psi) \to \delta \in Cn_{\mathbb{L}}(\expfilorp)$. By observation \ref{varphi_in_anel_agmpabd}, it is the case that $\varphi \lor \psi \in \expfilorp$. Therefore, by $\postp1$ and \textit{modus ponens}, we have $\delta \in \expfilorp$, as desired.
\par }

\noindent (ii) If $\varphi \in \expfilorp$, then $\expfip \subseteq \expfilorp$:

{\leftskip=1.5cm \noindent
Suppose $\varphi \in \expfilorp$. From this, we have that $\expfilorp \cup \{\varphi\} = \expfilorp$. By \textit{monotonicity} and $\postp1$, we have $Cn_{\mathbb{L}}(\expfilorp \cup \{\varphi\}) = Cn_{\mathbb{L}}(\expfilorp) = \expfilorp$. Since $\expfip \subseteq Cn_{\mathbb{L}}(\expfilorp \cup \{\varphi\})$, by $\postp6$, then, as required, $\expfip \subseteq \expfilorp$.
\par }
\end{proof}

\begin{properties} \label{proposition_propriedades_teta7_agmpabd}
The following properties are consequences of postulate $\postp7$ in the presence of $\postp1$ - $\postp5$:
\begin{description}

\leftskip 40pt
    \item (i) $\expfilorp \subseteq \expfip$ or $\expfilorp \subseteq \exppsip$
    \item (ii) If $\varphi \notin \expfilorp$, then $\expfilorp \in \exppsip$
\end{description}
\end{properties}
\begin{proof} \mbox{}\\
(i) $\expfilorp \subseteq \expfip$ or $\expfilorp \subseteq \exppsip$:

{\leftskip=1.5cm \noindent
By observation \ref{varphi_in_anel_agmpabd}, we have that $\varphi \lor \psi \in \expfilorp$. By postulate $\postp4$, we therefore have that $\til (\varphi \lor \psi) \notin \expfilorp$.
By $\postp1$, it is also the case that $\til \varphi \land \til \psi \notin \expfilorp$. From this, we have that it is not the case that both $\til\varphi \in \expfilorp$ and $\til\psi \in \expfilorp$ (otherwise, by postulate $\postp1$, we would have $\til\varphi \land \til\psi \in \expfilorp$, 
which contradicts $\til\varphi \land \til\psi \notin \expfilorp$). 
Therefore, either $\til\varphi \notin \expfilorp$ or $\til\psi \notin \expfilorp$. Hence, by postulate $\postp7$, we obtain $\expfilorp \subseteq \expfip$ or $\expfilorp \subseteq \exppsip$, as required.
\par }

\noindent (ii) If $\varphi \notin \expfilorp$, then $\expfilorp \subseteq \exppsip$:

{\leftskip=1.5cm \noindent
Assume $\varphi \notin \expfilorp$. By observation \ref{varphi_in_anel_agmpabd}, we have that $\varphi \lor \psi \in \expfilorp$. Thus, suppose by \textit{reductio} that $\til\psi \in \expfilorp$. In this case, since $\varphi \lor \psi, \til \psi \vdash_{\mathbb{L}}\varphi$, then, by postulate $\postp1$, $\varphi \in \expfilorp$, a contradiction with the initial assumption. Therefore, $\til\psi \notin \expfilorp$. Hence, by postulate $\postp7$, we obtain, as required, $\expfilorp \subseteq \exppsip$.
\par }
\end{proof}

\noindent\textbf{Observation \ref{remark_superconjunto_maximal_conjunto_de_crencas_agmpabd}}: \textit{Any $\Theta' \in \excedfi$ (definition \ref{definicao_conj_excedente_agmpabd}) is a belief set.}

\begin{proof}\mbox{}\\
    \noindent
     We need to show that, for any $\Theta' \in \excedfi$, it is the case that $\Theta' = Cn_{\mathbb{L}}(\Theta')$. Suppose, by \textit{reductio}, that there is some $\Theta' \in \excedfi$ such that $\Theta' \neq Cn_{\mathbb{L}}(\Theta')$. By \textit{inclusion}, we know that $\Theta' \subseteq Cn_{\mathbb{L}}(\Theta')$, but, by the \textit{reductio} assumption, we have $\Theta' \subset Cn_{\mathbb{L}}(\Theta')$ (a). Note that $Cn_{\mathbb{L}}(\Theta')$ satisfies conditions (i), (ii) and (iii) of definition \ref{definicao_supconj_maximal_agmpabd}, since (i) $\Theta \subset \Theta' \subset Cn_{\mathbb{L}}(\Theta')$, (ii) since $\Theta' \subset Cn_{\mathbb{L}}(\Theta')$, evidently, by set theory, $Cn_{\mathbb{L}}(\Theta') \cap H(\Theta, \varphi) \neq \emptyset$, and (iii) directly, $\bot \notin Cn_{\mathbb{L}}(\Theta')$. By condition (iv), however, there is no $\Theta'' \supset \Theta'$ that satisfies (i), (ii) and (iii). Therefore, we have a contradiction with (a). As required, $\Theta' = Cn_{\mathbb{L}}(\Theta')$.
\end{proof}

\

\noindent\textbf{Observation \ref{varphi_in_anel_conj_maximal_agmpabd}}: \textit{Let $\Theta' \in \excedfi$ be a maximal non-trivial \textbf{AGM}$p_{abd}$ superset of $\Theta$ with respect to $\varphi$ (definition \ref{definicao_supconj_maximal_agmpabd}). It is the case that $\varphi \in \Theta'\setminus\Theta$ for every $\Theta' \in \excedfi$.}

    \begin{proof}\mbox{}\\
    By definition \ref{definicao_supconj_maximal_agmpabd}, it is the case that $\papfig$. Thus, by lemma \ref{lema_sempre_existe_auto_explicacao_agmpabd}, $H(\Theta, \varphi) \neq \emptyset$. By condition (ii) of definition \ref{definicao_supconj_maximal_agmpabd}, therefore, there must be at least one abductive hypothesis $\alpha \in \Theta'$. Since, by condition (i), $\Theta \subset \Theta'$ and by lemma \ref{remark_superconjunto_maximal_conjunto_de_crencas_agmpabd}, it is the case that $\Theta'$ is a belief set, then $\varphi \in \Theta'$ (since, by condition (i) of definition \ref{conj_hip_abdutivas_agmpabd} of $H(\Theta, \varphi)$, it is the case that $\Theta, \alpha \vdash_{\mathbb{L}} \varphi$). Since, again, $\papfig$ is the case, then $\varphi \notin \Theta$. Hence, $\varphi \in \Theta'\setminus\Theta$ for every $\Theta' \in \excedfi$.
    \end{proof}

\

\noindent\textbf{Lemma \ref{lema_alpha_ou_fi_in_supconj_maximal_agmpabd}}: \textit{Let $\Theta' \in \excedfi$ be a maximal non-trivial \textbf{AGM}$p_{abd}$ superset of $\Theta$ with respect to $\varphi$ (definition \ref{definicao_supconj_maximal_agmpabd}). For $\alpha \in H(\Theta, \varphi)$, it is the case that $\alpha \lor \varphi \in \Theta'$ for every $\Theta' \in \excedfi$.}

\begin{proof} \mbox{}\\
    By observation \ref{varphi_in_anel_conj_maximal_agmpabd}, $\varphi \in \Theta' \setminus \Theta$, so, in particular, $\varphi \in \Theta'$. Since, by observation \ref{remark_superconjunto_maximal_conjunto_de_crencas_agmpabd}, $\Theta' = Cn_{\mathbb{L}}(\Theta')$, and it is the case that $\varphi \vdash_{\mathbb{L}} \alpha \lor \varphi$ for $\alpha \in H(\Theta, \varphi)$, then $\alpha \lor \varphi \in \Theta'$ for every $\Theta' \subseteq \excedfi$.    
\end{proof}

\


\noindent\textbf{Lemma \ref{lema_conjunto_dos_superconjuntos_maximais_nunca_vazio_agmpabd}}: \textit{Let $\excedfi$ be the surplus set of $\Theta$ with respect to $\varphi$ (definition \ref{definicao_conj_excedente_agmpabd}). It is the case that $\excedfi \neq \emptyset$.}

\begin{proof}\mbox{}\\
    \noindent
    By definition \ref{definicao_supconj_maximal_agmpabd}, it is the case that $\papfig$. Then, $\Theta$ is a belief set, $\Theta \neq \emptyset$, $\Theta \neq \Theta_\bot$ and $\til \varphi \notin \Theta$. Thus, $\Theta \cup \{\varphi\} \nvdash_{\mathbb{L}} \bot$. Since the set $\Theta \cup \{\varphi\}$ is consistent, then, by theorem \ref{teorema_lindenbaum_los}, of Lindenbaum-Łoś, there is a $\bot$-saturated set, i.e., a maximal non-trivial set $\Theta'$ with respect to $\bot$, such that $\varphi \in \Theta'$. Thus, taking conditions (i), (ii), (iii) and (iv) of definition \ref{definicao_supconj_maximal_agmpabd}, we evidently have that (i) $\Theta \subset \Theta'$; by corollary \ref{corolario_alpha_ou_fi_sempre_satisfaz_ii_supconj_maximal_agmpabd} and by the fact that maximal sets are closed under logical consequences, we have that $\alpha \lor \varphi \in \Theta' \cap H(\Theta, \varphi)$, for $\alpha \in H(\Theta, \varphi)$ and, therefore, (ii) $\Theta' \cap H(\Theta, \varphi) \neq \emptyset$; finally, by the Lindenbaum-Łoś theorem itself, we have that $\Theta'$ is consistent, that is, (iii) $\Theta' \nvdash_{\mathbb{L}} \bot$, and that it is maximal, i.e., (iv) there is no $\Theta'' \supset \Theta'$ that satisfies (i), (ii) and (iii). Therefore, we can conclude that there is always a $\Theta' \in \excedfi$, so that $\excedfi \neq \emptyset$.  
\end{proof}

\

\begin{lemma} \label{lema_conjuntos_explanatoriamente_excedentes_iguais_agmpabd}
Let $\excedfi$ and $\Theta \top \psi$ be the surplus sets of $\Theta$ with respect to $\varphi$ and $\psi$, respectively (definition \ref{definicao_conj_excedente_agmpabd}). It is the case that if $\vdash_{\mathbb{L}} \varphi \leftrightarrow \psi$, then $\excedfi = \Theta \top \psi$.
       
\end{lemma}

\begin{proof}\mbox{}\\
    \noindent
     Since, by definition, $\papfig$ and $\pappsig$ are the case, then, by lemma \ref{lema_conjunto_dos_superconjuntos_maximais_nunca_vazio_agmpabd}, it is always the case that $\excedfi \neq \emptyset$ and $\Theta \top \psi \neq \emptyset$. Therefore, this condition does not need to be demonstrated. 
     Consider $\vdash_{\mathbb{L}} \varphi \leftrightarrow \psi$. By lemma \ref{conj_hipoteses_abdutivas_iguais_varphi1_varphi2_agmpabd}, it is the case that $H(\Theta, \varphi) = H(\Theta, \psi)$. Now consider any $\Theta' \in \excedfi$, i.e., a $\Theta'$ that satisfies conditions (i)-(iv) of definition \ref{definicao_supconj_maximal_agmpabd}. Since $H(\Theta, \varphi) = H(\Theta, \psi)$, then, considering in particular condition (ii), $\Theta' \cap H(\Theta, \varphi) = \Theta' \cap H(\Theta, \psi)$. In other words, $\Theta' \in \Theta \top \psi$. Hence, $\excedfi = \Theta \top \psi$.
\end{proof}

\begin{remark}\label{remark_partial_meet_conjunto_de_crencas_agmpabd}
Let $\gamma$ be the function defined according to \ref{definicao_funcao_selecao_partial_meet_agmpabd}. It is the case that $\bigcap \gamma (\Theta, \varphi)$ is a belief set.
\end{remark}

\begin{proof}\mbox{}\\
    By definition \ref{definicao_funcao_selecao_partial_meet_agmpabd}, $\gamma (\Theta, \varphi) \subseteq \excedfi$ and, by lemma \ref{lema_conjunto_dos_superconjuntos_maximais_nunca_vazio_agmpabd}, $\excedfi \neq \emptyset$. Hence, $\bigcap \gamma (\Theta, \varphi) \neq \emptyset$ and this condition does not need to be demonstrated.   
    Thus, we need to show that $\bigcap \gamma (\Theta, \varphi) = Cn_{\mathbb{L}}(\bigcap \gamma (\Theta, \varphi))$. By observation \ref{remark_superconjunto_maximal_conjunto_de_crencas_agmpabd}, every $\Theta' \in \excedfi$ is a belief set, i.e., $\Theta' = Cn_{\mathbb{L}}(\Theta')$. Thus, we have that $\gamma (\Theta, \varphi) = \{\Theta', \Theta'', ...\}$. Consequently, $\bigcap \gamma (\Theta, \varphi) = \bigcap \{\Theta', \Theta'', ...\} = \Theta' \cap \Theta'' \cap ...$. Since deductive closure under $Cn$ is preserved under intersections, then $Cn_{\mathbb{L}}(\bigcap \gamma (\Theta, \varphi)) = Cn_{\mathbb{L}}(\Theta') \cap Cn_{\mathbb{L}}(\Theta'') \cap ... = \Theta' \cap \Theta'' \cap ... = \bigcap \gamma (\Theta, \varphi)$, as required.   
\end{proof}

\noindent\textbf{Theorem \ref{teorema_partial_meet_teta1-teta5_agmpabd}}: \textit{For every $\Theta$ and $\varphi$ such that $\papfig$, $\ooplusg$ is an \textbf{AGM}$p_{abd}$ partial meet abductive expansion operation of $\Theta$ with respect to $\varphi$---definition \ref{definicao_expansao_abdutiva_partial_meet_agmpabd}---if and only if $\ooplusg$ satisfies postulates $\postp1$ - $\postp5$ from definition \ref{definicao_postulados_expansao_abdutiva_agmpabd} of abductive expansion of $\Theta$ with respect to $\varphi$.}

\begin{proof} \mbox{}\\
Due to lemma \ref{lema_conjunto_dos_superconjuntos_maximais_nunca_vazio_agmpabd} and corollary \ref{corolario_alpha_ou_fi_sempre_expande_agmpabd}, the function $\ooplusg$, on both sides, always obtains a non-empty set. Hence, such conditions do not need to be considered.

\noindent (Construction $\Rightarrow$ postulates): \

\noindent $\postp1$ $\expfip$ is a belief set.

{\leftskip=1.5cm \noindent
    By definition \ref{definicao_expansao_abdutiva_partial_meet_agmpabd}, $\expfip = \bigcap \gamma (\Theta, \varphi)$ and by observation \ref{remark_partial_meet_conjunto_de_crencas_agmpabd}, $\bigcap \gamma (\Theta, \varphi)$ is a belief set. Therefore, $\expfip$ is a belief set.
\par }

\noindent $\postp2$ $\Theta \subset \expfip$.

{\leftskip=1.5cm \noindent
    By definition \ref{definicao_expansao_abdutiva_partial_meet_agmpabd}, $\expfip = \bigcap \gamma (\Theta, \varphi)$ and by definition \ref{definicao_funcao_selecao_partial_meet_agmpabd}, $\gamma (\Theta, \varphi) \subseteq \excedfi$. Since, by condition (i) of definition \ref{definicao_supconj_maximal_agmpabd}, for any $\Theta' \in \excedfi$, it is the case that $\Theta \subset \Theta'$, then $\Theta \subset \bigcap \gamma (\Theta, \varphi)$. Thus, $\Theta \subset \expfip$.
\par }

\noindent $\postp3$ $\expfip \cap H(\Theta,\varphi) \neq \emptyset$.

{\leftskip=1.5cm \noindent
    By definition \ref{definicao_expansao_abdutiva_partial_meet_agmpabd}, $\expfip = \bigcap \gamma (\Theta, \varphi)$ and by definition \ref{definicao_funcao_selecao_partial_meet_agmpabd}, $\gamma (\Theta, \varphi) \subseteq \excedfi$. By lemma \ref{lema_alpha_ou_fi_in_supconj_maximal_agmpabd}, for $\alpha \in H(\Theta, \varphi)$, it is the case that $\alpha \lor \varphi \in \Theta'$ for every $\Theta' \in \excedfi$. In particular, therefore, $\alpha \lor \varphi \in \Theta'$ for every $\Theta' \in \gamma (\Theta, \varphi)$. By lemma \ref{lema_alpha_ou_fi_in_hip_abdutivas_agmpabd}, if $\alpha \in H(\Theta, \varphi)$, it is the case that $\alpha \lor \varphi \in H(\Theta, \varphi)$. Therefore, $\alpha \lor \varphi \in \bigcap \gamma (\Theta, \varphi) \cap H(\Theta, \varphi) \neq \emptyset$. Hence, $\expfip \cap H(\Theta,\varphi) \neq \emptyset$.
\par }

\noindent $\postp4$ $\expfip \nvdash_{\mathbb{L}} \bot$.

{\leftskip=1.5cm \noindent
 By definition \ref{definicao_expansao_abdutiva_partial_meet_agmpabd}, $\expfip = \bigcap \gamma (\Theta, \varphi)$ and by definition \ref{definicao_funcao_selecao_partial_meet_agmpabd}, $\gamma (\Theta, \varphi) \subseteq \excedfi$. Suppose, by \textit{reductio}, that $\bigcap \gamma (\Theta, \varphi) \vdash_{\mathbb{L}} \bot$. In this case, for $\{\Theta', \Theta'', ...\} \subseteq \gamma (\Theta, \varphi)$, we have that $\Theta' \cap \Theta'' \cap... \vdash_{\mathbb{L}} \bot$. By observation \ref{remark_superconjunto_maximal_conjunto_de_crencas_agmpabd}, for every $\Theta' \in \excedfi$---in particular, for every $\Theta' \in \gamma (\Theta, \varphi) \subseteq \Theta\top\varphi$---it is the case that $\Theta' = Cn_{\mathbb{L}}(\Theta')$. Then, we have that $Cn_{\mathbb{L}}(\Theta') \cap Cn_{\mathbb{L}}(\Theta'') \cap... \vdash_{\mathbb{L}} \bot$. Thus, for all $\Theta' \in \gamma (\Theta, \varphi) \subseteq \Theta\top\varphi$, we have that $\bot \in Cn_{\mathbb{L}}(\Theta')$, a contradiction with condition (iii) of definition \ref{definicao_supconj_maximal_agmpabd}. Therefore, $\bigcap \gamma (\Theta, \varphi) \nvdash_{\mathbb{L}} \bot$ and, as required, $\expfip \nvdash_{\mathbb{L}} \bot$.  
\par }

\noindent $\postp5$ If $\vdash_{\mathbb{L}} \varphi \leftrightarrow \psi$, then $\expfip = \exppsip$.

{\leftskip=1.5cm \noindent
By definition \ref{definicao_expansao_abdutiva_partial_meet_agmpabd}, $\expfip = \bigcap \gamma (\Theta, \varphi)$.
Suppose $\vdash_{\mathbb{L}} \varphi \leftrightarrow \psi$. Then, by lemma \ref{lema_conjuntos_explanatoriamente_excedentes_iguais_agmpabd}, $\excedfi = \Theta \top \psi$. Since, by definition \ref{definicao_funcao_selecao_partial_meet_agmpabd}, $\gamma (\Theta, \varphi) \subseteq \excedfi$ and $\gamma$ is a function, then $\gamma(\Theta, \varphi) = \gamma(\Theta, \psi)$. Therefore, $\bigcap \gamma(\Theta, \varphi) = \bigcap \gamma(\Theta, \psi)$. Finally, then, $\expfip = \exppsip$.
\par }

\

\noindent (Postulates $\Rightarrow$ construction): \

Let $\ooplusg$ be a function that satisfies postulates $\postp1$ - $\postp5$. Let $\gamma^\ooplus$ be the following function:

 $$\gamma^\ooplus (\Theta,\varphi) = \{\Theta' : \Theta' \in \excedfi \ \text{and} \ \expfip \subseteq \Theta'\}$$

We need to show that (1) $\gamma^\ooplus$ is a selection function (definition \ref{definicao_funcao_selecao_partial_meet_agmpabd}) and (2) that $\expfip = \bigcap \gamma^\ooplus (\Theta,\varphi)$:

\noindent (1) $\gamma^\ooplus$ is a selection function:

{\leftskip=1.5cm \noindent
     We have that $\gamma^\ooplus(\Theta, \varphi) \subseteq \excedfi$ is immediate from the construction. By lemma \ref{lema_conjunto_dos_superconjuntos_maximais_nunca_vazio_agmpabd}, $\excedfi \neq \emptyset$. By postulates $\postp2$, $\postp3$ and $\postp4$, and conditions (i)-(iii) of definition \ref{definicao_conj_excedente_agmpabd}, we have that $\expfip \subseteq \Theta' \in \excedfi$ for some $\Theta' \in \excedfi$. By the definition of $\gamma^\ooplus$, $\Theta' \in \gamma^\ooplus(\Theta, \varphi)$. Hence, $\gamma^\ooplus(\Theta, \varphi) \neq \emptyset$.     
\par }

\noindent (2) $\expfip = \bigcap \gamma^\ooplus (\Theta,\varphi)$:

{\leftskip=1.5cm \noindent
    $\expfip \subseteq \bigcap \gamma^\ooplus (\Theta,\varphi)$ is directly satisfied by the definition of $\gamma^\ooplus$, since $\expfip \subseteq \Theta'$ for every $\Theta' \in \gamma^\ooplus(\Theta, \varphi)$. It is therefore only necessary to show that $\bigcap \gamma^\ooplus (\Theta,\varphi) \subseteq \expfip$.\\[1mm]    
    Suppose, by \textit{reductio}, that there exists a $\delta \in \bigcap\gamma^\ooplus(\Theta, \varphi)$, but $\delta \notin \expfip$. By the definition of $\gamma^\ooplus$, the sets $\Theta'$ belonging to $\gamma^\ooplus$ must satisfy: (a) $\Theta' \in \excedfi$ and (b) $\expfip \subseteq \Theta'$. First, let us analyze (b). By the Lindenbaum-Łoś theorem \ref{teorema_lindenbaum_los}, if $\delta \notin \expfip$ (the \textit{reductio} assumption), then there exists a $\delta$-saturated set $\bar\Theta$ (definition \ref{definicao_delta_saturado}) such that $\expfip \subseteq \bar\Theta \subseteq \mathcal{L}_{\Sigma_*}$. Therefore, the $\delta$-saturated set $\bar\Theta$ satisfies condition (b).\\[1mm]
    Now let us analyze condition (a) for $\bar\Theta$ according to criteria (i)-(iv) of definition \ref{definicao_conj_excedente_agmpabd}. By postulate $\postp2$, we directly have $\Theta \subset \expfip \subseteq \bar\Theta$ and thus (i) $\Theta \subset \bar\Theta$. By postulate $\postp3$, we have $\expfip \cap H(\Theta, \varphi) \neq \emptyset$. Since $\expfip \subseteq \bar\Theta$, then (ii) $\bar\Theta \cap H(\Theta, \varphi) \neq \emptyset$. Since $\bar\Theta$ is $\delta$-saturated, then (iii) $\bot \notin Cn_{\mathbb{L}}(\bar\Theta)$. Finally, since $\bar\Theta$ is maximal, then (iv) there is no $\Theta'' \supset \bar\Theta$ that satisfies (i), (ii) and (iii) (i.e., for any $\Theta'' \supset \bar\Theta$, $\bot \in Cn_{\mathbb{L}}(\Theta'')$). Hence, $\bar\Theta \in \excedfi$, which satisfies condition (a). Thus, $\bar\Theta \in \gamma^\ooplus(\Theta, \varphi)$.\\[1mm]  
    However, by the initial assumption, $\delta \in \bigcap\gamma^\ooplus(\Theta, \varphi)$, therefore $\delta \in \bar\Theta$. A contradiction with the fact that $\bar\Theta$ is $\delta$-saturated. Hence, $\expfip = \bigcap \gamma^\ooplus (\Theta,\varphi)$ as required.   
\par }
\end{proof}

\begin{remark} \label{remark_varphi1_in_excedente_entao_varphi_2_in_excedente}
Let $\excedfi$ and $\Theta \top \psi$ be surplus sets of $\Theta$ with respect to $\varphi$ and $\psi$, respectively (definition \ref{definicao_conj_excedente_agmpabd}). If $\Theta' \in \excedfi$ and $\psi \in \Theta'$, then $\Theta' \in \Theta \top \psi$.
\end{remark}
\begin{proof} \mbox{}\\
    By definitions \ref{definicao_supconj_maximal_agmpabd} and \ref{definicao_conj_excedente_agmpabd}, $\papfig$ and $\pappsig$ are the case. Suppose $\Theta' \in \excedfi$ and $\psi \in \Theta'$. Since $\Theta' \in \excedfi$, then: (i) $\Theta \subset \Theta'$, (ii) $\Theta' \cap H(\Theta, \varphi) \neq \emptyset$, (iii) $\bot \notin Cn_{\mathbb{L}}(\Theta')$, and (iv) $\Theta'$ is maximal.\\[1mm]  
    Since $\psi \in \Theta'$, by Lemma \ref{lema_alpha_ou_fi_in_supconj_maximal_agmpabd}, we have that $\alpha \lor \psi \in \Theta'$ for some $\alpha \in H(\Theta, \psi)$ (considering, by Lemma \ref{lema_sempre_existe_auto_explicacao_agmpabd}, that $H(\Theta, \psi) \neq \emptyset$). By Lemma \ref{lema_alpha_ou_fi_in_hip_abdutivas_agmpabd}, we have that $\alpha \lor \psi \in H(\Theta, \psi)$. Therefore, $\alpha \lor \psi \in \Theta' \cap H(\Theta, \psi)$, so $\Theta' \cap H(\Theta, \psi) \neq \emptyset$. Thus, $\Theta'$ satisfies conditions (i)-(iv) of Definition \ref{definicao_supconj_maximal_agmpabd} for $\psi$, i.e., $\Theta' \in \Theta \top \psi$.
 \end{proof}
    
\

\begin{lemma} \label{lema_excedente_ou_igual_excedente_uniao_B5_agmpabd}
Let $\excedfi$, $\Theta \top \psi$ and $\excedlor$ be surplus sets of $\Theta$ with respect to $\varphi$, $\psi$ and $\varphi \lor \psi$, respectively (definition \ref{definicao_conj_excedente_agmpabd}). It is the case that $\excedlor = \excedfi \cup \Theta \top \psi$.
\end{lemma}

\begin{proof} \mbox{}\\
By lemma \ref{lema_conjunto_dos_superconjuntos_maximais_nunca_vazio_agmpabd}, $\excedfi \neq \emptyset$, $\Theta \top \psi \neq \emptyset$ and $\excedlor \neq \emptyset$.\\
(1) $\excedlor \subseteq \excedfi \cup \Theta \top \psi$:

{\leftskip=1.5cm \noindent
    If $\Theta' \in \excedlor$, then $\Theta' \in \excedfi \cup \Theta \top \psi$. 
    Suppose $\Theta' \in \excedlor$. By observation \ref{varphi_in_anel_conj_maximal_agmpabd}, it is the case that $\varphi \lor \psi \in \Theta'$. Since $\Theta'$ is maximal, if $\varphi \notin \Theta'$, then $\til \varphi \in \Theta'$. In this case, since, by observation \ref{remark_superconjunto_maximal_conjunto_de_crencas_agmpabd}, $\Theta' = Cn_{\mathbb{L}}(\Theta')$, then $\psi \in \Theta'$. Analogously, if $\psi \notin \Theta'$, then $\til \psi \in \Theta'$ and $\varphi \in \Theta'$. Therefore, either $\varphi \in \Theta'$ or $\psi \in \Theta'$. By observation \ref{remark_varphi1_in_excedente_entao_varphi_2_in_excedente}, either $\Theta' \in \excedfi$ or $\Theta' \in \Theta \top \psi$. Hence, as required, $\Theta' \in \excedfi \cup \Theta \top \psi$.
\par }

\noindent (2) $\excedfi \cup \Theta \top \psi \subseteq \excedlor$:

{\leftskip=1.5cm \noindent
     If $\Theta' \in \excedfi \cup \Theta \top \psi$, then $\Theta' \in \excedlor$.
     Suppose $\Theta' \in \excedfi \cup \Theta \top \psi$. Then $\Theta' \in \excedfi$ or $\Theta' \in \Theta \top \psi$. Take the first disjunct. If $\Theta' \in \excedfi$, then, by observation \ref{varphi_in_anel_conj_maximal_agmpabd}, $\varphi \in \Theta'$ and by observation \ref{remark_superconjunto_maximal_conjunto_de_crencas_agmpabd}, $\varphi \lor \psi \in \Theta'$. In this case, by observation \ref{remark_varphi1_in_excedente_entao_varphi_2_in_excedente}, $\Theta' \in \excedlor$. The same reasoning can be applied to the second disjunct $\Theta' \in \Theta \top \psi$. Hence, as required, $\Theta' \in \excedlor$.
\par }
\end{proof}

\

\begin{lemma} \label{lema_gamma_excedente_subseteq_gamma_excedente_varphi1_varphi2_B6_agmpabd}
Let $\excedfi$ be the surplus set of $\Theta$ with respect to $\varphi$ (definition \ref{definicao_conj_excedente_agmpabd}) and $\gamma_\leqslant$ be a selection function established by the (relational and transitive) marking-off identity $\leqslant$ (definition \ref{definicao_marking_off_agmpabd}). It is the case that if $\excedfi \cap \gamma_\leqslant(\Theta, \varphi \lor \psi ) \neq \emptyset$, then $\gamma_\leqslant(\Theta, \varphi) \subseteq \gamma_\leqslant(\Theta, \varphi \lor \psi )$.
\end{lemma}

\begin{proof} \mbox{}\\
Consider $\excedfi \cap \gamma_\leqslant(\Theta, \varphi \lor \psi ) \neq \emptyset$. We need to show that if $\Theta' \in \gamma_\leqslant(\Theta, \varphi)$, then $\Theta' \in \gamma_\leqslant(\Theta, \varphi \lor \psi )$. \\[1mm]
Suppose $\Theta' \in \gamma_\leqslant(\Theta, \varphi)$ and, by \textit{reductio}, $\Theta' \notin \gamma_\leqslant(\Theta, \varphi \lor \psi )$. By definition \ref{definicao_marking_off_agmpabd} and part (2) of lemma \ref{lema_excedente_ou_igual_excedente_uniao_B5_agmpabd}, we have that $\Theta' \in \gamma_\leqslant(\Theta, \varphi) \subseteq \excedfi \subseteq \excedlor$. Therefore, we have that $\Theta' \in \excedlor$, but, by the \textit{reductio} hypothesis, $\Theta' \notin \gamma_\leqslant(\Theta, \varphi \lor \psi)$. In this case, by definition \ref{definicao_marking_off_agmpabd} of relational selection function, there exists $\Theta'' \in \excedlor$ such that $\Theta' < \Theta''$ (strictly better), i.e., $\Theta' \leqslant \Theta''$ but $\Theta'' \nleqslant \Theta'$ (a). Since, by the initial hypothesis, $\excedfi \cap \gamma_\leqslant(\Theta, \varphi \lor \psi ) \neq \emptyset$, then there is at least one $\Theta^\# \in \excedfi$ such that $\Theta^\# \in \gamma_\leqslant(\Theta, \varphi \lor \psi )$. Since $\Theta^\#$ is $\leqslant$-best in $\excedlor$ (i.e., $\Theta^\# \in \gamma_\leqslant(\Theta, \varphi \lor \psi)$), we have that $\Theta'' \leqslant \Theta^\#$ (b). Since $\Theta' \in \gamma_\leqslant(\Theta, \varphi)$ (i.e., $\Theta'$ is $\leqslant$-best in $\excedfi$) and $\Theta^\# \in \excedfi$, we have that $\Theta^\# \leqslant \Theta'$ (c). By the transitivity of $\leqslant$, we have, from (b) and (c), $\Theta'' \leqslant \Theta^\# \leqslant \Theta'$. A contradiction with (a). Hence, $\Theta' \in \gamma_\leqslant(\Theta, \varphi \lor \psi )$.    
\end{proof}

\

\begin{lemma} \label{lema_gamma_excedente_subseteq_bigcap_gamma_excedente_B7_agmpabd}
Let $\gamma_\leqslant$ be a selection function established by the marking-off identity $\leqslant$ (definition \ref{definicao_marking_off_agmpabd}). If $\gamma_\leqslant(\Theta, \varphi) \subseteq \gamma_\leqslant(\Theta, \psi)$, then $\bigcap \gamma_\leqslant(\Theta, \psi) \subseteq \bigcap \gamma_\leqslant(\Theta, \varphi)$.
\end{lemma}

\begin{proof} \mbox{}\\
Consider $\gamma_\leqslant(\Theta, \varphi) \subseteq \gamma_\leqslant(\Theta, \psi)$. We need to show that if $\delta \in \bigcap \gamma_\leqslant(\Theta, \psi)$, then $\delta \in \bigcap \gamma_\leqslant(\Theta, \varphi)$. \\[1mm]
Suppose $\delta \in \bigcap \gamma_\leqslant(\Theta, \psi)$. Then, $\delta \in \Theta'$ for every $\Theta' \in \gamma_\leqslant(\Theta, \psi)$. Since $\gamma_\leqslant(\Theta, \varphi) \subseteq \gamma_\leqslant(\Theta, \psi)$, then $\delta \in \Theta'$ for every $\Theta' \in \gamma_\leqslant(\Theta, \varphi)$. Therefore, $\delta \in \bigcap \gamma_\leqslant(\Theta, \varphi)$.
\end{proof}

\noindent\textbf{Lemma \ref{lema_partial_meet_relacional_teta6_agmpabd}}: \textit{Any relational partial meet abductive expansion function, i.e., $\expfip = \bigcap \gamma_\leqslant(\Theta,\varphi)$ (definitions \ref{definicao_expansao_abdutiva_partial_meet_agmpabd} and \ref{definicao_marking_off_agmpabd}), satisfies postulate $\postp6$ from definition \ref{definicao_postulados_expansao_abdutiva_agmpabd}.}

\begin{proof} \mbox{}\\
Let $\ooplusg$ be a relational \textbf{AGM}$p_{abd}$ partial meet abductive expansion function. We need to show that $\ooplusg$ satisfies postulate $\postp6 \ \expfip \subseteq Cn_{\mathbb{L}}(\expfilorp \cup \{\varphi\})$. By definition \ref{definicao_expansao_abdutiva_partial_meet_agmpabd}, it is the case that $\expfip = \bigcap \gamma_\leqslant(\Theta, \varphi)$ and that $\expfilorp = \bigcap \gamma_\leqslant(\Theta, \varphi \lor \psi)$. Therefore, according to postulate $\postp6$, we need to show that $\bigcap \gamma_\leqslant(\Theta, \varphi) \subseteq Cn_{\mathbb{L}}(\bigcap \gamma_\leqslant(\Theta, \varphi \lor \psi) \cup \{\varphi\})$. Suppose, then, that $\delta \in \bigcap \gamma_\leqslant(\Theta, \varphi)$. We need to show that $\delta \in Cn_{\mathbb{L}}(\bigcap \gamma_\leqslant(\Theta, \varphi \lor \psi) \cup \{\varphi\})$. \\[1mm]
To do this, consider any $\Theta' \in \gamma_\leqslant(\Theta, \varphi \lor \psi)$. By definition \ref{definicao_marking_off_agmpabd} of the marking-off identity, $\Theta'$ is $\leqslant$-best in $\excedlor$, i.e., $\Theta'' \leqslant \Theta'$ for every $\Theta'' \in \excedlor$ (a). It is also the case that either $\varphi \notin \Theta'$ or $\varphi \in \Theta'$. Consider the first disjunct. In this case, since $\Theta'$ is maximal, then $\til \varphi \in \Theta'$. Since, by observation \ref{remark_superconjunto_maximal_conjunto_de_crencas_agmpabd}, $\Theta' \in Th(\mathbb{L})$, then $\til \varphi \lor \delta \in \Theta'$ and, since $\vdash_{\mathbb{L}} \til \varphi \lor \delta \leftrightarrow \varphi \to \delta$, then $\varphi \to \delta \in \Theta'$ (b). \\[1mm]
In the case of the second disjunct, $\varphi \in \Theta'$. Since $\Theta' \in \gamma_\leqslant(\Theta, \varphi \lor \psi) \subseteq \excedlor$ and $\varphi \in \Theta'$, then by observation \ref{remark_varphi1_in_excedente_entao_varphi_2_in_excedente} (applied with $\varphi \lor \psi$ in the role of $\varphi$ and $\varphi$ in the role of $\psi$), we have that $\Theta' \in \excedfi$. Consider some arbitrary $\Theta^\# \in \excedfi$. Since, by lemma \ref{lema_excedente_ou_igual_excedente_uniao_B5_agmpabd}, $\excedfi \subseteq \excedlor$, then $\Theta^\# \in \excedlor$. Considering (a), we have, by definition \ref{definicao_marking_off_agmpabd}, that $\Theta^\# \leqslant \Theta'$. Therefore, since $\Theta^\# \leqslant \Theta'$ for every $\Theta^\# \in \excedfi$ (and $\Theta^\#$ was taken arbitrarily), we have that $\Theta'$ is $\leqslant$-best in $\excedfi$, i.e., $\Theta' \in \gamma_\leqslant(\Theta, \varphi)$. Since, by hypothesis, $\delta \in \bigcap \gamma_\leqslant(\Theta, \varphi)$, and $\Theta' \in \gamma_\leqslant(\Theta, \varphi)$, then $\delta \in \Theta'$. Thus, since $\Theta' \in Th(\mathbb{L})$, then $\varphi \to \delta \in \Theta'$ (c). From (b) and (c), we have that in both disjuncts, it is the case that $\varphi \to \delta \in \Theta'$ for every $\Theta' \in \gamma_\leqslant(\Theta, \varphi \lor \psi)$. Therefore, $\varphi \to \delta \in \bigcap \gamma_\leqslant(\Theta, \varphi \lor \psi)$. By \textit{inclusion}, we have that $\varphi \to \delta \in Cn_{\mathbb{L}}(\bigcap \gamma_\leqslant(\Theta, \varphi \lor \psi))$ and, by the \textit{deduction theorem}, we have, as required, that $\delta \in Cn_{\mathbb{L}}(\bigcap \gamma_\leqslant(\Theta, \varphi \lor \psi) \cup \{\varphi\})$.
\end{proof}

\

\noindent\textbf{Lemma \ref{lema_partial_meet_transitivamente_relacional_teta7_agmpabd}}: \textit{Any transitively relational partial meet abductive expansion function (definitions \ref{definicao_expansao_abdutiva_partial_meet_agmpabd} and \ref{definicao_marking_off_agmpabd}) satisfies postulate $\postp7$ from definition \ref{definicao_postulados_expansao_abdutiva_agmpabd}.}

\begin{proof} \mbox{}\\
Let $\ooplusg$ be a transitively relational partial meet abductive expansion function. We need to show that $\ooplusg$ satisfies postulate $\postp7 \ \text{If } \til \varphi \notin \expfilorp$, then $\expfilorp \subseteq \expfip$. \\[1mm]
Suppose $\til \varphi \notin \expfilorp$. By definition \ref{definicao_expansao_abdutiva_partial_meet_agmpabd}, it is the case that $\expfip = \bigcap \gamma_\leqslant(\Theta, \varphi)$ and that $\expfilorp = \bigcap \gamma_\leqslant(\Theta, \varphi \lor \psi)$. Since, by assumption, $\til \varphi \notin \expfilorp$, then $\til \varphi \notin \bigcap \gamma_\leqslant(\Theta, \varphi \lor \psi)$. Thus, there is a $\Theta' \in \gamma_\leqslant(\Theta, \varphi \lor \psi)$ such that $\til \varphi \notin \Theta'$. Since $\Theta'$ is maximal, then $\varphi \in \Theta'$. By observation \ref{remark_varphi1_in_excedente_entao_varphi_2_in_excedente}, we have that $\Theta' \in \excedfi$ and, therefore, that $\Theta' \in \excedfi \cap \gamma_\leqslant(\Theta, \varphi \lor \psi)$. Thus, $\excedfi \cap \gamma_\leqslant(\Theta, \varphi \lor \psi) \neq \emptyset$ and, by lemma \ref{lema_gamma_excedente_subseteq_gamma_excedente_varphi1_varphi2_B6_agmpabd}, we have that $\gamma_\leqslant(\Theta, \varphi) \subseteq \gamma_\leqslant(\Theta, \varphi \lor \psi )$. Therefore, by lemma \ref{lema_gamma_excedente_subseteq_bigcap_gamma_excedente_B7_agmpabd}, we have that $\bigcap \gamma_\leqslant(\Theta, \varphi \lor \psi) \subseteq \bigcap \gamma_\leqslant(\Theta, \varphi)$. Hence, as desired, $\expfilorp = \bigcap \gamma_\leqslant(\Theta, \varphi \lor \psi) \subseteq \bigcap \gamma_\leqslant(\Theta, \varphi) = \expfip$.    
\end{proof}


\noindent\textbf{Theorem \ref{teorema_Tabd_partial_meet_transitivamente_relacional_teta1-teta7_agmpabd}}: \textit{For every $\Theta$ and $\varphi$ such that $\papfig$, $\ooplusg$ is a transitively relational \textbf{AGM}$p_{abd}$ \textit{partial meet} abductive expansion operation (definitions \ref{definicao_expansao_abdutiva_partial_meet_agmpabd} and \ref{definicao_marking_off_agmpabd}) if, and only if, $\ooplusg$ satisfies postulates $\postp1$ - $\postp7$ from definition \ref{definicao_postulados_expansao_abdutiva_agmpabd} of abductive expansion on $\Theta$.}

\begin{proof} \mbox{}\\
Due to lemma \ref{lema_conjunto_dos_superconjuntos_maximais_nunca_vazio_agmpabd} and corollary \ref{corolario_alpha_ou_fi_sempre_expande_agmpabd}, the function $\ooplusg$, on both sides, always obtains a non-empty set. Hence, such conditions do not need to be considered.

\noindent (Construction $\Rightarrow$ postulates): \

Result already obtained by theorem \ref{teorema_partial_meet_teta1-teta5_agmpabd} 
and lemmas \ref{lema_partial_meet_relacional_teta6_agmpabd} and 
\ref{lema_partial_meet_transitivamente_relacional_teta7_agmpabd}.

\noindent (Postulates $\Rightarrow$ construction): \

Theorem \ref{teorema_partial_meet_teta1-teta5_agmpabd}---postulates $\Rightarrow$ construction direction---already shows that, assuming postulates $\postp1$ - $\postp5$ from definition \ref{definicao_postulados_expansao_abdutiva_agmpabd}, $\gamma^\ooplus$ is a selection function and that $\expfip = \bigcap \gamma^\ooplus(\Theta, \varphi)$. Thus, we need to show that $\gamma^\ooplus$ is transitively relational.

\noindent We define $\leqslant$ on all maximal non-trivial supersets of $\Theta$ as follows:
      
            {\leftskip=1.5cm \noindent
            For all $\Theta',\Theta'' \in Th(\mathbb{L})$, $\Theta'' \leqslant \Theta'$ if and only if the following three conditions are satisfied:
            \begin{description}
                \setlength{\parskip}{3pt}
                \leftskip 40pt
                \item (i) $\Theta'' \in \excedfi$ for some $\varphi \in \mathcal{L}_{\Sigma_*}$;
                \item (ii) $\Theta' \in \excedfi$ and $\expfip \subseteq \Theta'$ for some $\varphi \in \mathcal{L}_{\Sigma_*}$;
                \item (iii) For every $\varphi \in \mathcal{L}_{\Sigma_*}$, if $\Theta',\Theta'' \in \excedfi$ and $\expfip \subseteq \Theta''$, then $\expfip \subseteq \Theta'$.
            \end{description}
            \par }
\noindent We then define the relational selection function $\gamma^\ooplus_\leqslant$ as the completion of $\gamma^\ooplus$:
$$\gamma^\ooplus_\leqslant(\Theta, \varphi) = \{\Theta' \in \excedfi: \bigcap \gamma^\ooplus(\Theta, \varphi) \subseteq \Theta'\} \text{ for all } \varphi \in \mathcal{L}_{\Sigma_*}.$$
         
Thus defined, the function $\gamma^\ooplus_\leqslant$ is a selection function---since $\gamma^\ooplus_\leqslant(\Theta, \varphi) \subseteq \gamma^\ooplus(\Theta, \varphi) \subseteq \excedfi$ and $\gamma^\ooplus_\leqslant(\Theta, \varphi) \neq \emptyset$---and determines the same \textit{partial meet} abductive expansion operation as $\gamma^\ooplus$---since $\bigcap \gamma^\ooplus_\leqslant(\Theta, \varphi) = \bigcap \gamma^\ooplus(\Theta, \varphi) = \expfip$. With this in mind, we need to show that:
            \begin{description}
                \setlength{\parskip}{3pt}
                \leftskip 40pt
                \item (1) The relation $\leqslant$ is relational with respect to $\gamma^\ooplus_\leqslant$, i.e., it satisfies the Marking-Off Identity (definition \ref{definicao_marking_off_agmpabd});
                \item (2) The relation $\leqslant$ is transitive with respect to $\gamma^\ooplus_\leqslant$, for all $\varphi \in \mathcal{L}_{\Sigma_*}$.
            \end{description}
            
\
         
         \noindent (1) The relation $\leqslant$ is relational with respect to $\gamma^\ooplus_\leqslant$ (satisfies the Marking-Off Identity)\footnote{Recall definition \ref{definicao_marking_off_agmpabd} (p. \pageref{definicao_marking_off_agmpabd}). The Marking-Off Identity with respect to $\gamma^\ooplus_\leqslant$ is given by: $\gamma^\ooplus_\leqslant(\Theta, \varphi) = \{\Theta' \in \excedfi : \Theta'' \leqslant \Theta' \text{ for all } \Theta'' \in \excedfi\}$.}:\\[1mm]        
         To show that $\leqslant$ satisfies the Marking-Off Identity, we need to verify two cases: (1A) if $\Theta' \in \gamma^\ooplus_\leqslant(\Theta, \varphi)$ and $\Theta'' \in \excedfi$, then $\Theta'' \leqslant \Theta'$; (1B) if $\Theta' \notin \gamma^\ooplus_\leqslant(\Theta, \varphi)$ and $\Theta' \in \excedfi$, then $\Theta' \prec \Theta''$ for some $\Theta'' \in \excedfi$.\\[2mm]        
         To prove (1A), suppose $\Theta' \in \gamma^\ooplus_\leqslant(\Theta, \varphi)$ and $\Theta'' \in \excedfi$. We need to show that $\Theta'' \leqslant \Theta'$, according to conditions (i)-(iii) of $\leqslant$ established above. By assumption, condition (i) $\Theta'' \in \excedfi$ for some $\varphi$ is immediately satisfied. Since $\Theta' \in \gamma^\ooplus_\leqslant(\Theta, \varphi)$, then, by the definition of $\gamma^\ooplus_\leqslant$, it is the case that $\Theta' \in \excedfi$ and $\bigcap \gamma^\ooplus(\Theta, \varphi) \subseteq \Theta'$, with $\expfip = \bigcap \gamma^\ooplus(\Theta, \varphi)$. Hence, condition (ii) $\Theta' \in \excedfi$ and $\expfip \subseteq \Theta'$ for some $\varphi$ is also satisfied. It remains to verify condition (iii).\\[2mm]
         Let $\psi \in \mathcal{L}_{\Sigma_*}$ and suppose $\Theta'', \Theta' \in \excedpsi$ and $\exppsip \subseteq \Theta''$---the antecedent of condition (iii) in the definition of $\leqslant$. To obtain $\Theta'' \leqslant \Theta'$, we must therefore prove its consequent, i.e., that $\exppsip \subseteq \Theta'$. More precisely, we need to show that if $\delta \in \exppsip$, then $\delta \in \Theta'$.\\[2mm]
         To do this, consider proposition \ref{proposition_propriedades_teta7_agmpabd} (i) $\expfilorp \subseteq \expfip$ or $\expfilorp \subseteq \exppsip$. Taking the first disjunct and the recently obtained condition (ii), we have $\expfilorp \subseteq \expfip \subseteq \Theta'$ for $\Theta' \in \excedfi$. Taking the second disjunct, the assumption that $\exppsip \subseteq \Theta''$ and condition (i) just obtained, we get $\expfilorp \subseteq \exppsip \subseteq \Theta''$ for $\Theta'' \in \excedfi$. In this latter case, however, we have that $\til \varphi \notin \expfilorp$ (otherwise, $\Theta'' \notin \excedfi$, since, by Observation \ref{varphi_in_anel_conj_maximal_agmpabd}, $\varphi \in \Theta''$ and thus we would violate definition \ref{definicao_conj_excedente_agmpabd} (iii)). Thus, by postulate $\postp7$, we have $\expfilorp \subseteq \expfip$. Hence, from both disjuncts, we obtain $\expfilorp \subseteq \expfip \subseteq \Theta'$ for $\Theta' \in \excedfi$. Let us call this result (a). \\[1mm]
         Finally, to show that if $\delta \in \exppsip$, then $\delta \in \Theta'$, suppose $\delta \in \exppsip$. By postulate $\postp6$, it is the case that $\exppsip \subseteq Cn_{\mathbb{L}}(\expfilorp \cup \{\psi\})$, so $\delta \in Cn_{\mathbb{L}}(\expfilorp \cup \{\psi\})$. By the \textit{deduction theorem}, we have $\psi \to \delta \in Cn_{\mathbb{L}}(\expfilorp)$ and, by $\postp1$, $\psi \to \delta \in \expfilorp$. Thus, by result (a), we have $\psi \to \delta \in \Theta'$. But, by assumption, $\Theta' \in \excedpsi$ and, by Observation \ref{varphi_in_anel_conj_maximal_agmpabd}, $\psi \in \Theta'$. Therefore, by Observation \ref{remark_superconjunto_maximal_conjunto_de_crencas_agmpabd} and \textit{modus ponens}, we have $\delta \in \Theta'$. Hence, it is the case that $\exppsip \subseteq \Theta'$ and, thus, condition (iii) of the definition of $\leqslant$ is finally satisfied and, as required, $\Theta'' \leqslant \Theta'$. That is, case (1A) is proven. \\[1mm]
         Consider case (1B), i.e., if $\Theta' \notin \gamma^\ooplus_\leqslant(\Theta, \varphi)$ and $\Theta' \in \excedfi$, then $\Theta' \prec \Theta''$ for some $\Theta'' \in \excedfi$ (i.e., $\Theta' \leqslant \Theta''$, but $\Theta'' \npreceq \Theta'$).   
         Let, therefore, $\Theta'' \in \gamma^\ooplus_\leqslant(\Theta, \varphi)$. Consider only condition (iii) of the definition of $\leqslant$. Since $\gamma^\ooplus_\leqslant(\Theta, \varphi) \subseteq \excedfi$, we have $\Theta', \Theta'' \in \excedfi$. Since $\Theta'' \in \gamma^\ooplus_\leqslant(\Theta, \varphi)$, by the definition of the completion function $\gamma^\ooplus_\leqslant$, we have $\bigcap \gamma^\ooplus(\Theta, \varphi) \subseteq \Theta''$ and, therefore, $\expfip \subseteq \Theta''$. But $\Theta' \notin \gamma^\ooplus_\leqslant(\Theta, \varphi)$. Therefore, by the definition of the completion function $\gamma^\ooplus_\leqslant$, we have $\bigcap \gamma^\ooplus(\Theta, \varphi) \nsubseteq \Theta'$ and, therefore, $\expfip \nsubseteq \Theta'$, which fails to satisfy condition (iii) of the definition of $\leqslant$. Hence, $\Theta' \leqslant \Theta''$, but $\Theta'' \npreceq \Theta'$, as required.

        \noindent (2) The relation $\leqslant$ is transitive with respect to $\gamma^\ooplus_\leqslant$, for all $\varphi \in \mathcal{L}_{\Sigma_*}$.

        Suppose $\Theta''' \leqslant \Theta''$ and $\Theta'' \leqslant \Theta'$ (Assumption A). We need to show that $\Theta''' \leqslant \Theta'$, that is, that conditions (i)-(iii) of the definition of $\leqslant$ are satisfied to obtain $\Theta''' \leqslant \Theta'$. Since $\Theta''' \leqslant \Theta''$, we have that $\Theta''' \in \Theta \top \delta$, for some $\delta \in \mathcal{L}_{\Sigma_*}$. Condition (i) is therefore satisfied. Similarly, since $\Theta'' \leqslant \Theta'$, we have that $\Theta' \in \Theta \top \delta$ and $\Theta^\ooplus_\delta \subseteq \Theta'$, for some $\delta \in \mathcal{L}_{\Sigma_*}$. Hence, condition (ii) is also satisfied. We must therefore verify condition (iii) of the definition of $\leqslant$ for $\Theta''' \leqslant \Theta'$. \\[1mm]
        Suppose, accordingly, as per condition (iii), for $\psi \in \mathcal{L}_{\Sigma_*}$, that $\Theta''', \Theta' \in \Theta \top \psi$ and $\exppsip \subseteq \Theta'''$ (Assumption B). We need to show that $\exppsip \subseteq \Theta'$. By lemma \ref{lema_excedente_ou_igual_excedente_uniao_B5_agmpabd}, we have that $\excedlor = \excedfi \cup \Theta \top \psi$. Therefore, since by Assumption B, $\Theta''', \Theta' \in \Theta \top \psi$, we have that $\Theta''', \Theta' \in \excedlor$. Since, by Assumption A, we have $\Theta''' \leqslant \Theta''$, then, by condition (ii) of the definition of $\leqslant$, we have that $\Theta'' \in \excedfi$ and $\expfip \subseteq \Theta''$ for some $\varphi \in \mathcal{L}_{\Sigma_*}$ (result (b)). Therefore, it is also the case that $\Theta'' \in \excedlor$. Thus, we have $\Theta''', \Theta'', \Theta' \in \excedlor$ (result (c)).\\[1mm]
        By proposition \ref{proposition_propriedades_teta7_agmpabd} (i), we have that $\expfilorp \subseteq \expfip$ or $\expfilorp \subseteq \exppsip$. Taking the first disjunct and result (b), we have $\expfilorp \subseteq \expfip \subseteq \Theta''$ and, therefore, $\expfilorp \subseteq \Theta''$. Since, by Assumption A, we have $\Theta'' \leqslant \Theta'$, and by result (c), we have $\Theta'', \Theta' \in \excedlor$, then, by condition (iii) of the definition of $\leqslant$, we obtain $\expfilorp \subseteq \Theta'$. Similarly, taking the second disjunct and Assumption B, we have $\expfilorp \subseteq \exppsip \subseteq \Theta'''$. Since, by Assumption A, $\Theta''' \leqslant \Theta''$ and, by result (c), $\Theta''', \Theta'' \in \excedlor$, then, by condition (iii) of the definition of $\leqslant$, we obtain $\expfilorp \subseteq \Theta''$. But, once again, since $\Theta'' \leqslant \Theta'$ by Assumption A, and $\Theta'', \Theta' \in \excedlor$ by result (c), then, by condition (iii) of $\leqslant$, we have $\expfilorp \subseteq \Theta'$. Therefore, from either disjunct, it is the case that $\expfilorp \subseteq \Theta'$ (result (d)).\\[1mm]
        Finally, to show that $\exppsip \subseteq \Theta'$, suppose $\delta \in \exppsip$. We need to show that $\delta \in \Theta'$. By postulate $\postp6$, we have $\exppsip \subseteq Cn_{\mathbb{L}}(\expfilorp \cup \{\psi\})$. Therefore, $\delta \in Cn_{\mathbb{L}}(\expfilorp \cup \{\psi\})$ and, by the \textit{deduction theorem}, $\psi \to \delta \in Cn_{\mathbb{L}}(\expfilorp)$ and, by $\postp1$, $\psi \to \delta \in \expfilorp$. Hence, by result (d), $\psi \to \delta \in \Theta'$. Since, by Assumption B, $\Theta' \in \Theta \top \psi$, then, by observation \ref{varphi_in_anel_conj_maximal_agmpabd}, $\psi \in \Theta'$. Therefore, by observation \ref{remark_superconjunto_maximal_conjunto_de_crencas_agmpabd} and \textit{modus ponens}, we have $\delta \in \Theta'$, as required. Thus, condition (iii) of $\leqslant$ for $\Theta''' \leqslant \Theta'$ is satisfied and, consequently, we obtain the transitivity of $\leqslant$.
\end{proof}

\end{appendices}

\bibliographystyle{plain}  
\bibliography{referencies} 

\end{document}